\documentclass{article}

     \PassOptionsToPackage{numbers, compress}{natbib}

     \usepackage[final]{neurips_2020}

\usepackage[utf8]{inputenc} 
\usepackage[T1]{fontenc}    
\usepackage{hyperref}       
\usepackage{url}            
\usepackage{booktabs}       
\usepackage{amsfonts,amsmath,amsfonts,amssymb,amsthm, mathtools}       
\usepackage{nicefrac}       
\usepackage{microtype}      
\usepackage{color}
\usepackage[ruled,vlined]{algorithm2e}
\usepackage{threeparttable}
\usepackage[separate-uncertainty=true, table-format = 2.2(2), table-align-uncertainty=true,detect-weight=true,mode=text]{siunitx}
\usepackage{etoolbox}
\usepackage{multirow}

\newtheorem{theorem}{Theorem}
\newtheorem{lemma}{Lemma}
\newtheorem*{remark}{Remark}

\newtheorem{assump}{Assumption}
\newtheorem{proposition}{Proposition}

\newcommand{\eps}{\varepsilon}
\newcommand{\R}{\mathbb{R}}
\newcommand{\tr}{\text{Tr}}
\newcommand{\mc}{\text{mc}}
\newcommand{\KL}{\text{KL}}
\newcommand{\mE}{\mathcal{E}}
\newcommand{\mL}{\mathcal{L}}
\newcommand{\mA}{\mathcal{A}}

\newcommand{\upk}{\overline{\kappa}}
\newcommand{\downk}{\underline{\kappa}}

\newcommand{\Q}{\mathcal{Q}}
\newcommand*\samethanks[1][\value{footnote}]{\footnotemark[#1]}

\newcommand*{\bbE}{\mathbb{E}}

\DeclareMathOperator*{\argmin}{arg\,min}
\DeclareMathOperator*{\argmax}{arg\,max}

\title{Spike and slab variational Bayes for high dimensional logistic regression}

\author{%
  Kolyan Ray\thanks{Equal contribution.}\\
  Department of Mathematics\\
  Imperial College London\\
  \texttt{kolyan.ray@ic.ac.uk}
   \And
   Botond Szab\'o\samethanks\\
   Department of Mathematics\\
   Vrije Universiteit Amsterdam\\
   \texttt{b.t.szabo@vu.nl}
   \And
   Gabriel Clara\\
   Department of Mathematics\\
   Vrije Universiteit Amsterdam\\
   \texttt{g.clara@student.vu.nl}
}

\begin{document}

\maketitle

\begin{abstract}
Variational Bayes (VB) is a popular scalable alternative to Markov chain Monte Carlo for Bayesian inference. We study a mean-field spike and slab VB approximation of widely used Bayesian model selection priors in sparse high-dimensional logistic regression. We provide non-asymptotic theoretical guarantees for the VB posterior in both $\ell_2$ and prediction loss for a sparse truth, giving optimal (minimax) convergence rates. Since the VB algorithm does not depend on the unknown truth to achieve optimality, our results shed light on effective prior choices. We confirm the improved performance of our VB algorithm over common sparse VB approaches in a numerical study.
\end{abstract}

\section{Introduction}

Let $x\in \R^p$ denote a feature vector and $Y\in \{0,1\}$ an associated binary label to be predicted. In logistic regression, one of the most widely used methods in classification, we model
\begin{equation}\label{model}
P(Y = 1|X = x) = 1-P(Y=0|X=x)=\Psi(x^T \theta) = \frac{e^{x^T\theta}}{1+e^{x^T\theta}},
\end{equation}
where $\theta \in \R^p$ is an unknown regression parameter and $\Psi(t) = e^t /(1+e^t)$ is the logistic function. 
Suppose we observe $n$ training examples $\{(x_1,y_1),\dots,(x_n,y_n)\}$.

We study the \textit{sparse high-dimensional} setting, where $n\leq p$ and typically $n\ll p$, and many of the coefficients of $\theta$ are (close to) zero. This setting has been studied by many authors, notably using $\ell_1$-regularized $M$-estimators (e.g. the LASSO), see for instance \cite{bach2010,li2015,negahban2009,negahban2012} and the references therein. In Bayesian logistic regression, one assigns a prior distribution to $\theta$, giving a probabilistic model. An especially natural Bayesian way to model sparsity is via a \textit{model selection} prior, which assigns probabilistic weights to every potential model, i.e. every subset of $\{1,\dots,p\}$ corresponding to selecting the non-zero coordinates of $\theta \in \R^p$. This is a widely used Bayesian approach and includes the hugely popular spike and slab prior \cite{george1993,mitchell1988}.

Such priors work well in many settings for estimation and prediction \cite{atchade2017,castillo2015,castillo2012}, uncertainty quantification \cite{ray2017,castillo2020} and multiple hypothesis testing \cite{castillo2018}, see \cite{bcg:review:2020} for a recent review. An especially attractive property is their interpretability, particularly for variable selection, compared to many other black-box machine learning methods. For example, such methods provide posterior inclusion probabilities of particular features, and their credible sets, which are often important in practice.

However, the discrete nature of such priors makes scalable computation hugely challenging. Under a model selection prior, posterior exploration typically involves searching over all $2^p$ possible models, making standard Markov chain Monte Carlo (MCMC) methods infeasible for moderate $p$ unless the feature vectors $\{x_1,\dots,x_n\}$ satisfy strong structural conditions like orthogonality \cite{castillo2012,vanerven2018}. There has been recent progress on adapting MCMC methods to sparse high-dimensional logistic regression \cite{narisetty2019}, while another common alternative is to instead use continuous shrinkage-type priors \cite{carvalho2010,wei2020}.

A popular scalable alternative is variational Bayes (VB), which approximates the posterior by solving an optimization problem. One proposes an approximating family of tractable distributions, called the variational family, and finds the member of this family that is closest to the computationally intractable posterior in Kullback-Leibler (KL) sense. This member is taken as a substitute for the posterior. An especially popular family consists of factorizable distributions, called \textit{mean-field variational Bayes}. VB scales to large data sets and works empirically well in many models, see \cite{blei2017} for a recent review.

In this work, we study the theoretical properties of a mean-field VB approach to sparsity-inducing priors with variational family the set of factorizable spike and slab distributions. This is a natural approximation since it keeps the discrete model selection aspect and many of the interpretable features of the original posterior, while reducing the full $O(2^p)$ model complexity to a much more manageable $O(p)$. The procedure is \textit{adaptive} in that it does not depend on the typically unknown sparsity level, which avoids delicate issues about tuning hyperparameters. This sparse variational family has been employed in various settings \cite{huang2016,logsdon2010,ormerod2017,rayszabo2019,titsias2011}, including logistic regression \cite{carbonetto2012,zhang2019}. VB is natural in model \eqref{model} since in even the simplest low-dimensional setting ($p \ll n$) using Gaussian priors, the posterior is intractable and VB is widely used \cite{bishop2006,jaakkola2000,paisley2012,titsias2014,wang2013}.

However, VB generally comes with few theoretical guarantees, with none currently available in high-dimensional logistic regression. Our main contribution is to show that the sparse VB posterior converges to the true sparse vector at the optimal (minimax) rate in both $\ell_2$ and prediction loss. We prove this under the same conditions for which the true (computationally infeasible) posterior is known to converge \cite{atchade2017}, showing that one does not necessarily need to sacrifice theoretical guarantees when using sparse VB, at least for estimation and prediction. Our convergence bounds are non-asymptotic and thus reflect relevant finite-sample situations.

Our results also provide practical insights on effective prior and VB calibrations, in particular the choice of prior slab distribution. Many existing works employ Gaussian slabs for the underlying prior, even though these cause excessive shrinkage and suboptimal parameter recovery in benchmark models \cite{castillo2012}. Our theoretical results show that optimal parameter recovery is possible if the VB posterior is based on a prior with heavier-tailed Laplace slabs, corroborating findings in linear regression that one should use prior slabs with exponential or heavier tails \cite{castillo2015,castillo2012}, including for VB \cite{rayszabo2019}. We confirm in simulations that using Laplace prior slabs, as our theory suggests, indeed empirically outperforms the usual VB choice of Gaussian prior slabs, demonstrating the practical importance of the prior slab choice. We further demonstrate that our VB algorithm is empirically competitive with other state-of-the-art Bayesian sparse variable selection methods for logistic regression.

Lastly, we provide conditions on the design matrix under which sparse VB can be expected to work well. Together, these provide theoretical backing for using VB for estimation and prediction in the widely used sparse high-dimensional Bayesian logistic regression model \eqref{model}.

\textbf{Related work}. Theoretical guarantees for VB have been studied for specific models, including linear models \cite{ormerod2017,rayszabo2019}, exponential family models \cite{wang2004,wang2004b}, stochastic block models \cite{zhang2017}, latent Gaussian models \cite{sheth2017} and topic models \cite{ghorbani2018}. In low dimensional settings $(p\ll n)$, some Bernstein-von Mises results have also been obtained \cite{lu2016,wang2019b,wang2019}. In high-dimensional and nonparametric settings, general results have been derived \cite{pati2017,zhang2017b} based on the classic Bayesian prior mass and testing approach \cite{ghosal2000}. There has also been work on studying variational approximations to fractional posteriors \cite{alquier2017,yang2017}. For logistic regression, theoretical results have been established for the fully Bayesian spike and slab approach \cite{atchade2017,narisetty2019} and its continuous relaxation \cite{wei2020}.

Theoretical guarantees for VB in sparse linear regression have recently been obtained in \cite{rayszabo2019}. We combine ideas from this paper with tools from high-dimensional and nonparametric Bayesian statistics \cite{atchade2017,castillo2015,nicklray2019} to obtain theoretical results in the nonlinear logistic regression model \eqref{model}. For our algorithm derivation, we use ideas from VB for Bayesian logistic regression \cite{carbonetto2012,jaakkola2000}.

\textbf{Organization}. In Section \ref{sec:setup} we detail the problem setup, including the notation, prior, variational family and conditions on the design matrix. Main results are found in Section \ref{sec:results}, details of the VB algorithm in Section \ref{sec:algorithm}, simulations in Section \ref{sec:num:anal} and discussion in Section \ref{sec:discussion}. In the supplement, we present streamlined proofs for the asymptotic results (Section \ref{sec:proofs}), additional simulations (Section \ref{sec:suppl:sim}), discussion concerning the design matrix conditions (Section \ref{sec:extra_design}), full statements and proofs of the non-asymptotic results (Section \ref{sec:suppl:proofs}) and a derivation of the VB algorithm (Section \ref{sec:algorithm_derivation}).

\section{Problem setup}\label{sec:setup}

\textbf{Notation}. Recall that we observe $n$ training examples $\{(x_1,y_1),\dots,(x_n,y_n)\}$ from model \eqref{model}. For $u\in \R^d$, we write $\|u\|_2 = (\sum_{i=1}^d |u_i|^2)^{1/2}$ for the usual Euclidean norm.
Let $X$ be the $n\times p$ matrix with $i^{th}$ row equal to $x_i^T = (x_{i1},\dots,x_{ip})$ and for $X_{\cdot j}$ the $j^{th}$ column of $X$, set
\begin{equation*}\label{X norm}
\|X\| = \max_{1\leq j \leq p}\|X_{\cdot j}\|_2 = \max_{1\leq j \leq p} (X^T X)_{jj}^{1/2}.
\end{equation*}
We denote by $P_{\theta}$ the probability distribution of observing $Y = (Y_1,\dots,Y_n)$ from model \eqref{model} with parameter $\theta\in\R^p$ and by $E_\theta$ the corresponding expectation. For two probability measures $P,Q$, we write $\KL(P||Q) = \int \log \tfrac{dP}{dQ}dP$ for the Kullback-Leibler divergence. Let $\nabla_\theta f(y,\theta)$ denote the gradient of $f$ with respect to $\theta$. For $\theta\in \R^p$ and a subset $S\subseteq \{1,\dots,p\}$, we write $|S|$ for the cardinality of $S$ and $\theta_S$ for the vector $(\theta_i : i\in S)\in \R^{|S|}$. We set $S_\theta = \{ 1 \leq i \leq p: \theta_i \neq 0\}$ to be the set of non-zero coefficients of $\theta$ and write $s_0 = |S_{\theta_0}|$ for the sparsity level of the true parameter $\theta_0$. Throughout the paper we work under the following frequentist assumption:
\begin{assump}\label{assump:freq}
There is a true $s_0$-sparse parameter $\theta_0\in \R^p$ generating the data $Y\sim P_{\theta_0}$.
\end{assump}

\subsection{Model selection priors and the variational approximation}

A model selection prior first selects a dimension $s \in \{0,\dots,p\}$ from a prior $\pi_p$, then a subset $S \subseteq \{1,\dots,p\}$ uniformly at random from all ${p\choose s}$ subsets of size $|S|=s$, and lastly selects a set of non-zero values for $\theta_S = (\theta_j)_{j\in S} \in \R^{|S|}$ from a product of centered Laplace distributions with density $\prod_{j\in S}  \text{Lap}(\lambda)(\theta_j) = (\lambda/2)^{|S|} \exp(-\lambda\sum_{j\in S}|\theta_j|)$, $\lambda >0$. This prior is represented via the following hierarchical scheme:
\begin{equation}\label{prior}
\begin{split}
s \sim \pi_p(s), \\
S | |S|=s \sim \text{Unif}_{p,s},\\
\theta_j \stackrel{ind}{\sim} \begin{cases} 
      \text{Lap}(\lambda), & j\in S, \\
      \delta_0, & j \not\in S,
   \end{cases}
\end{split}
\end{equation}
where $\text{Unif}_{p,s}$ selects $S$ from the $p\choose s$ possible subsets of $\{1,\dots,p\}$ of size $s$ with equal probability and $\delta_0$ denotes the Dirac mass at zero. We assume that there are constants $A_1,A_2,A_3,A_4>0$ with
\begin{align}\label{prior_cond}
A_1 p^{-A_3} \pi_p (s-1) \leq \pi_p (s) \leq A_2 p^{-A_4} \pi_p (s-1), \quad \quad s= 1,\dots,p.
\end{align}
Condition \eqref{prior_cond} is satisfied by a wide range of priors, such as complexity priors \cite{castillo2012} and binomial priors, including the widely used spike and slab prior $\theta_j \sim^{iid} \rho \text{Lap}(\lambda) + (1-\rho)\delta_0$ by taking $\pi_p=\text{Binomial}(p,\rho)$. Assigning a hyperprior $\rho\sim \text{Beta}(1,p^t)$, $t>1$, to the prior inclusion probabilities also satisfies \eqref{prior_cond} (\cite{castillo2012}, Example 2.2), allows mixing over the sparsity level and most importantly gives a spike and slab prior calibration that does not depend on unknown hyperparameters.

While most works use Gaussian slabs for the prior \cite{carbonetto2012,huang2016,logsdon2010,ormerod2017,titsias2011,zhang2019}, we instead use Laplace slabs since using slab distributions with lighter than exponential tails can cause excessive shrinkage and deteriorate estimation in linear regression \cite{castillo2012,rayszabo2019}. We illustrate numerically that the same phenomenon can occur in logistic regression, where using Laplace rather than Gaussian prior slabs improves estimation, see Section \ref{sec:num:anal}. This shows that our theoretical results in Section \ref{sec:results} are reflected in practice and sheds light on suitable prior choices. Full details of the modified algorithm are provided in Algorithm \ref{algorithm} below.

For our theoretical results, we suppose the regularization parameter $\lambda$ of the slab satisfies
\begin{equation}\label{lambda}
2 \|X\| \sqrt{\log p} \leq  \lambda \leq  \alpha \|X\| \sqrt{\log p}
\end{equation}
for some $\alpha\geq 2$. The choice $\lambda \asymp \|X\| \sqrt{\log p}$ is common for the regularization parameter of the LASSO (\cite{buhlmann2011}, Chapter 6), which corresponds to the posterior mode based on a full product Laplace prior $\theta_j \sim^{iid} \text{Lap}(\lambda)$ with no extra model selection as in \eqref{prior}. We note additional model selection is necessary, since the pure Laplace prior can behave badly in sparse high-dimensional settings \cite{castillo2012}. Specific values of $\|X\|$ for some design matrices are given in Section \ref{sec:extra_design}, but one should typically think of $\|X\| \sim \sqrt{n}$.

Bayesian inference about $\theta_0$, including reconstruction of the class probabilities $P_{\theta_0}(Y=1|X=x)$, is carried out via the posterior distribution $\Pi(\cdot|Y)$.
Since computing the posterior is infeasible for large $p$, we consider the following mean-field family of approximating distributions
\begin{equation}\label{VarFamily}
\Q = \left\{ Q_{\mu,\sigma,\gamma} =\prod_{j=1}^p \left[ \gamma_j N(\mu_j,\sigma_j^2) + (1-\gamma_j) \delta_0 \right]: ~~\gamma_j \in [0,1],~~ \mu_j \in \R,~~ \sigma_j^2 >0  \right\}.
\end{equation}
The VB posterior is the element of $\Q$ that minimizes the KL divergence to the exact posterior
\begin{equation}\label{VB}
Q^* = \argmin_{Q\in \Q} \KL(Q||\Pi(\cdot|Y)).
\end{equation}
The family $\Q$ consists of all factorizable distributions of spike and slab form, which is a natural approximation for sparse settings with variable selection. The $(\gamma_j)$ correspond to the VB variable inclusion probabilities, thereby keeping the interpretability of the original model selection prior. While the prior may factorize like \eqref{VarFamily}, the posterior does not, and we replace the full $2^p$ posterior model weights with the $p$ probabilities $(\gamma_j)$, greatly reducing the posterior dimension. Note that while the \textit{prior} has Laplace slabs, we can fit Gaussian distributions in the variational family since the likelihood induces subgaussian tails in the \textit{posterior}.

Computing the VB posterior $Q^*$ for the variational family $\Q$ in \eqref{VarFamily} is an optimization problem that has been studied in the literature \cite{huang2016,logsdon2010,ormerod2017,rayszabo2019,titsias2011}, including for logistic regression \cite{carbonetto2012,zhang2019}, mainly using coordinate ascent variational inference (CAVI). While these works mostly consider Gaussian slabs for the prior, CAVI can be suitably modified to the Laplace case, see Algorithm \ref{algorithm}.

\subsection{Design matrix and sparsity assumptions}\label{sec:design matrix}

In the high-dimensional case $p>n$, the parameter $\theta$ in model \eqref{model} is not identifiable, let alone estimable, without additional conditions on the design matrix $X$.  In the sparse setting, a sufficient condition for consistent estimation is `local invertibility' of $X^TX$ when restricted to sparse vectors. The following definitions are taken from \cite{atchade2017} and make precise this notion of invertibility. Define the diagonal matrix $W\in \R^{n\times n}$ with $i^{th}$ diagonal entry
\begin{equation}\label{W matrix}
W_{ii} = g''(x_i^T \theta_0) = \Psi(x_i^T\theta_0) (1-\Psi(x_i^T\theta_0))
\end{equation}
and the compatibility type constant 
$$\underline{\kappa} = \inf \left\{ \frac{\|W^{1/2}X\theta\|_2^2}{\|X\|^2 \|\theta\|_2^2} : \|\theta_{S_0^c}\|_1 \leq 7\|\theta_{S_0}\|_1, \theta \neq 0 \right\}.$$
For dimension $s\in \{1,\dots,p\}$, set
$$\overline{\kappa}(s) = \sup \left\{ \frac{\|X\theta\|_2^2}{\|X\|^2 \|\theta\|_2^2}: 0 \neq |S_\theta| \leq s \right\}, \qquad \underline{\kappa}(s) = \inf \left\{ \frac{\|W^{1/2}X\theta\|_2^2}{\|X\|^2 \|\theta\|_2^2}: 0 \neq |S_\theta| \leq s \right\}.$$
For a given $L>0$ and $\alpha$ defined in \eqref{lambda}, we require the following bound on the design matrix
\begin{equation}
\|X\| \geq \alpha \max \left(\frac{50(L+2)\|X\|_\infty}{\underline{\kappa}((L+1)s_0)}, \frac{64}{3 \underline{\kappa}}\right)  s_0\sqrt{\log p}.\label{cond:design}
\end{equation}
These constants are widely used in the sparsity literature (e.g.  \cite{buhlmann2011}), including for high-dimensional logistic regression \cite{atchade2017,negahban2012,wei2020}. Assuming such constants are bounded away from zero and infinity, Atchad\'e \cite{atchade2017} proves that the original posterior $\Pi(\cdot|Y)$ converges to the truth at the optimal rate. We show here that under no further assumptions on the design matrix $X$, the VB posterior $Q^*$ also converges to the truth at the optimal rate. We thus provide theoretical guarantees for the scalable VB approximation under the same conditions for which the true posterior is known to converge.

Many standard design matrices satisfy these compatibility conditions, such as orthogonal designs, i.i.d. (including Gaussian) random matrices and matrices satisfying the `strong irrepresentability condition' of \cite{zhao2006}. Details of these examples and further discussion are provided in Section \ref{sec:extra_design} in the supplement (see also Chapter 6 of \cite{buhlmann2011}). 

For a normalized design matrix with entries of size $O(1$), one has $\|X\| \sim \sqrt{n}$, so that \eqref{cond:design} is a minimal sample size condition. For suitably bounded compatibility constants, this translates into the minimal sample size $n \gtrsim s_0^2 \log p$, as in \cite{atchade2017}. The frequentist $\ell_1$-regularized $M$-estimator is known to converge at the same rate under similar assumptions to ours for a deterministic design matrix $X$ \cite{li2015} and under slightly weaker sample size conditions for an i.i.d. subgaussian random design matrix $X$ ($n \gtrsim s_0 \log p$) \cite{negahban2012}.

\section{Main results}\label{sec:results}

We now provide theoretical guarantees for the VB posterior $Q^*$ in \eqref{VB}. We present here our results in a simpler asymptotic form as $n,p\to\infty$ for easier readability. More complicated, but practically more relevant, finite sample guarantees are provided in Section \ref{sec:suppl:proofs} of the supplement. In particular, one should keep in mind that the results here do indeed reflect finite-sample behaviour.

We investigate how well the VB posterior recovers the true underlying high-dimensional parameter $\theta_0$. This is measured via the speed of posterior concentration, which studies the size of the \textit{smallest} $\ell_2$ or prediction type neighbourhood around the true $\theta_0$ that contains most of the (VB) posterior probability \cite{ghosal2000}. This is a frequentist assessment that describes the typical behaviour of the VB posterior under the true generative model, see Assumption \ref{assump:freq}. 

Posterior concentration rates are now entering the machine learning community as tools to gain insights into (variational) Bayesian methods and assess the suitability of priors and their calibrations (e.g. \cite{pati2017,polson:2018,szabo2019,wang2019}). Such results also quantify the typical distance between a point estimator $\hat{\theta}$ (posterior mean/median) and the truth (\cite{ghosal2000}, Theorem 2.5), as well as the typical posterior spread about the truth. Taken together, these quantities are crucial for the accuracy of Bayesian uncertainty quantification and so good posterior concentration results are necessary conditions for ensuring the latter. Ideally, most of the posterior probability should be concentrated in a ball centered around the true $\theta_0$ with radius proportional to the optimal (minimax) rate. This is the case for the true computationally infeasible posterior \cite{atchade2017} and extends to the VB posterior $Q^*$, as we now show. This provides a universal, objective guarantee for the VB posterior.

\begin{theorem}\label{thm:contraction}
Suppose the model selection prior \eqref{prior} satisfies \eqref{prior_cond} and \eqref{lambda} and the design matrix $X$ satisfies assumption \eqref{cond:design} for some sequence $L=M_n$. Then the VB posterior $Q^*$ satisfies
\begin{align*}
E_{\theta_0}Q^*\left(\theta\in\mathbb{R}^p:\, \|\theta-\theta_0\|_2 \geq  \frac{M_n^{1/2}}{\underline\kappa\big(M_n s_0\big)} \frac{\sqrt{s_0\log p}}{\|X\| }\right)= O\big(C_{\kappa}/M_n \big)+o(1),
\end{align*}
where $C_{\kappa}=L_0\big( \frac{  \bar\kappa(L_0s_0)}{\downk((1+4L_0/A_4)s_0)^{2}}+\downk (L_0s_0)^{-1}\big)$ and $L_0=2\max\{ A_4/5,(1.1 + 4\alpha^2/\underline{\kappa}+2A_4 +\log(4+\upk (s_0)))/A_4\}$. Define the mean-squared prediction error $\|p_\theta-p_0\|_n^2=\tfrac{1}{n} \sum_{i=1}^n(\Psi(x_i^T\theta)- \Psi(x_i^T\theta_0))^2$, where we recall $P_\theta(Y=1|X=x) = \Psi(x_i^T\theta)$. Then the VB posterior $Q^*$ satisfies 
\begin{align*}
E_{\theta_0}Q^*\left(\theta\in\mathbb{R}^p:\, \|p_\theta-p_0\|_n \geq  \frac{\sqrt{M_n\overline\kappa(M_n s_0)}}{\underline\kappa\big(M_n s_0\big)} \sqrt{\frac{s_0\log p}{n} }\right)=  O\big(C_\kappa/M_n\big)+o(1).
\end{align*}
In particular, if $C_\kappa/M_n \to 0$, then the posterior concentrates around the true sparse parameter $\theta_0$ at the optimal (minimax) rate in both $\ell_2$ and mean-squared prediction loss.
\end{theorem}

\begin{remark}
In Theorem \ref{thm:contraction}, we keep track of the compatibility numbers in the rate. If $ \overline\kappa(L_0s_0)$, $\downk((1+4L_0/A_4)s_0)$ and $\downk (L_0s_0)$ are bounded away from zero and infinity, as is often the case, the right hand side of both displays tends to zero for any $M_n\to \infty$ growing arbitrary slowly. In this case, the rates simplify to $M_n^{1/2} \sqrt{s_0 \log p}/\|X\|$ and $M_n \sqrt{s_0 \log p/n}$, respectively. In Section \ref{sec:extra_design} of the supplement, we give sufficient conditions for this to happen.
\end{remark}

Since $\sqrt{s_0\log p}/\|X\|$ is the minimax rate, this result says that for estimating $\theta_0$, in either $\ell_2$ or prediction loss, the VB approximation behaves optimally from a theoretical frequentist perspective. In fact, since these conditions are essentially the same as were used to study the true posterior \cite{atchade2017}, our results suggest one does not lose much by using this computationally more efficient approximation, at least regarding estimation and prediction. This backs up the empirical evidence that VB can provide excellent scalable estimation.For a non-asymptotic version of Theorem \ref{thm:contraction}, see Theorem \ref{thm:contraction:nonasymp} in Section \ref{sec:suppl:proofs}.

Since the prior and variational family do not depend on unknown parameters (e.g. sparsity level $s_0=|S_{\theta_0}|$ of $\theta_0$), the procedure is \textit{adaptive}, i.e. it can recover an $s_0$-sparse truth nearly as well as if we knew the exact true sparsity level beforehand. This avoids difficult issues about tuning parameter selection. As mentioned above, posterior concentration results can be used to obtain guarantees for point estimators, such as the VB posterior mean.


\begin{theorem}\label{thm:convergence:postmean}
Assume the conditions of Theorem \ref{thm:contraction} and that $n\geq 1.1 + 4\alpha^2/\underline{\kappa}+2A_4 +\log(4+\upk (s_0)))$. Then the VB predictive mean $\hat{p}^*(x) = \int P_\theta(Y=1|X=x)dQ^*(\theta)=\int \Psi(x^T\theta) dQ^*(\theta)$ and true prediction function $p_0(x) = P_{\theta_0}(Y=1|X=x)$ satisfy
\begin{align*}
 P_{\theta_0}\left(\|\hat{p}^*-p_0\|_n \geq  \frac{M_n^{1/2} \overline{\kappa}( \frac{2n}{A_4\log p}+ s_0)^{1/2}}{\underline{\kappa}( \frac{2n}{A_4\log p} +(1-2/A_4)s_0)}\sqrt{ \frac{s_0\log p}{n }}\right) = O(C_\kappa/M_n) + o(1).
\end{align*} 
\end{theorem}

\begin{remark}
If $\|X\| \sim \sqrt{n}$, as is often the case, then the extra sample size condition in Theorem \ref{thm:convergence:postmean} is automatically implied by \eqref{cond:design}, see Section \ref{sec:design matrix}. We have again kept track of the compatibility numbers; if these are bounded away from zero and infinity, then the probability converges to zero and we recover the rate $M_n^{1/2} \sqrt{s_0 \log p/n}$ for any $M_n \to \infty$ arbitrarily slowly.
\end{remark}

We next show the VB posterior $Q^*$ does not provide overly conservative model selection in the sense that it concentrates on models of size at most a constant multiple of the true model size $s_0$. This gives some guarantees for variable selection, in particular bounding the number of false positives. It provides a first theoretical underpinning for interpretable inference when using this VB approximation.

\begin{theorem}\label{thm:selection}
Under the conditions of Theorem \ref{thm:contraction}, the VB posterior satisfies
\begin{align*}
E_{\theta_0}Q^*(\theta\in\mathbb{R}^p:\, |S_{\theta}| \geq M_ns_0 ) = O\big(C_{\kappa}/M_n \big)+o(1).
\end{align*}
\end{theorem}

We have thus shown the VB posterior $Q^*$ (1) concentrates at the optimal rate around the sparse truth in both $\ell_2$ and prediction loss and (2) does not select overly large models, with the VB posterior mean sharing property (1). These provide some reassuring theoretical guarantees regarding the behaviour of this scalable and interpretable sparse VB approximation. Our results are reflected in practice, as our VB algorithm performs better empirically than commonly used sparse VB approaches, see Section \ref{sec:num:anal}.

\section{Variational algorithm}\label{sec:algorithm}

We present a coordinate-ascent variational inference (CAVI) algorithm to compute the VB posterior $Q^*$ in \eqref{VB}. Consider the prior \eqref{prior} with $\theta_j \sim^{iid} (1-w)\delta_0+w\text{Lap}(\lambda)$ and hyperprior $w\sim \text{Beta}(a_0,b_0)$. Introducing binary latent variables $(z_j)_{j=1}^p$, this spike and slab prior has hierarchical representation
\begin{equation}\label{prior_beta}
\begin{split}
w & \sim \text{Beta}(a_0,b_0), \\
z_j | w & \sim^{iid} \text{Bernoulli}(w),\\
\theta_j|z_j & \sim^{ind} (1-z_j)\delta_0+ z_j \text{Lap}(\lambda) .
\end{split}
\end{equation}
Minimizing the objective \eqref{VB} is intractable for Bayesian logistic regression, so we instead minimize a surrogate objective obtained by maximizing a lower bound on the marginal likelihood following the ideas of \cite{jaakkola2000,carbonetto2012}, see Section \ref{sec:algorithm_derivation} for details. This common approach is known to lead to improved accuracy in approximation \cite{jaakkola2000}, see also Chapter 10.6 of \cite{bishop2006}. The surrogate objective is non-convex, as is typically the case in VB, so the CAVI algorithm can be sensitive to initialization \cite{blei2017} and parameter updating order \cite{rayszabo2019}. Introducing a free parameter $\eta\in \R^n$, we can establish the upper bound
\begin{equation}\label{VB_modified}
\KL (Q_{\mu,\sigma,\gamma}||\Pi(\cdot|Y)) \leq \KL(Q_{\mu,\sigma,\gamma}||\Pi) - E_{\theta\sim Q_{\mu,\sigma,\gamma}}[f(\theta,\eta)]
\end{equation}
for a suitable function $f$ defined in \eqref{lower_bound}. We minimize the right-hand side over the variational family $Q_{\mu,\sigma,\gamma} \in \Q$, i.e. over the parameters $\mu,\sigma,\gamma$. Since we seek the tightest possible upper bound in \eqref{VB_modified}, we also minimize this over the free parameter $\eta$. In particular, CAVI alternates between updating $\eta$ for fixed $\mu,\sigma,\gamma$ and then cycling through $\mu_j$, $\sigma_j$, $\gamma_j$ and updating these while keeping all other parameters fixed. Keeping all other parameters fixed, one updates $\mu_j$ and $\sigma_j$ by minimizing
\begin{align}\label{CAVI_mu}
\mu_j \mapsto \quad & \lambda \sigma_j \sqrt{\frac{2}{\pi}} e^{-\frac{\mu_j ^2}{2\sigma_j^2}} + \lambda\mu_j \mathrm{erf}\bigg( \frac{\mu_j}{\sqrt{2}\sigma_j}\bigg) + \mu_j^2 \sum_{i=1}^n \frac{1}{4\eta_i} \tanh(\eta_i/2)  x_{ij}^2 \nonumber \\
& + \mu_j \bigg( \sum_{i=1}^n \frac{1}{2\eta_i} \tanh(\eta_i/2) x_{ij} \sum_{k\neq j} \gamma_k x_{ik} \mu_k  - \sum_{i=1}^n (y_i-1/2)x_{ij} \bigg), \\
\sigma_j \mapsto \quad & \lambda \sigma_j \sqrt{\frac{2}{\pi}} e^{-\frac{\mu_j^2}{2\sigma_j^2}} + \lambda\mu_j \mathrm{erf}\left( \frac{\mu_j}{\sqrt{2}\sigma_j}\right) -\log \sigma_j +\sigma_j^2  \sum_{i=1}^n \frac{1}{4\eta_i} \tanh(\eta_i/2)  x_{ij}^2, \nonumber
\end{align}			
respectively, where $\mathrm{erf}(x) = 2/\sqrt{\pi} \int_{0}^x e^{-t^2} dt$ is the error function. One updates $\gamma_j$ by solving
\begin{align}\label{CAVI_alpha}
 - \log \frac{\gamma_j}{1-\gamma_j} & = \log \frac{b_0}{a_0} +\lambda \sigma_j \sqrt{\frac{2}{\pi}} e^{-\frac{\mu_j^2}{2\sigma_j^2}} + \lambda\mu_j \mathrm{erf}\left( \frac{\mu_j}{\sqrt{2}\sigma_j}\right) - \mu_j \sum_{i=1}^n (y_i-1/2)x_{ij} \nonumber - \frac{1}{2} \\
& \quad + \sum_{i=1}^n \frac{1}{4\eta_i} \tanh(\eta_i/2) \bigg( x_{ij}^2 (\mu_j^2 + \sigma_j^2) + 2x_{ij} \mu_j \sum_{k\neq j} \gamma_k x_{ik} \mu_k \bigg) - \log (\lambda\sigma_j) 
\end{align}
and updates $\eta$ via
\begin{equation}\label{CAVI_eta}
\eta_i^2 =  E_{\mu,\gamma,\sigma}(x_i^T \theta)^2 =  \sum_{k = 1}^{p} \gamma_k x_{ik}^2(\mu_k^2 + \sigma_k^2) + \sum_{k = 1}^{p}\sum_{l\neq k} (\gamma_kx_{ik}\mu_k)(\gamma_l x_{il}\mu_l).
\end{equation}
A full derivation of \eqref{VB_modified}-\eqref{CAVI_eta} can be found in Section \ref{sec:algorithm_derivation}  In our implementation, we minimize \eqref{CAVI_mu} using the limited memory Broyden-Fletcher-Goldfarb-Shanno (BFGS) algorithm \cite{liu1989}. To improve scalability, we also perform the parameter updates in parallel, updating the objective function parameter values only at the end of each iteration. Details are provided in Algorithm \ref{algorithm}.

	\begin{algorithm}[H]\label{algorithm}
		\SetAlgoLined
		\KwIn{$X\in\R^{n\times p}, Y\in \{0,1\}^{p}$}
		$(\mu,\sigma,\gamma) \gets \mathtt{init\_param}()$\\
		\While{$!\mathtt{convergence}$}{
			$\eta \gets \mathtt{update\_lower\_bound}(\mu, \sigma, \gamma)$
			\tcp*{implements eq.\eqref{CAVI_eta}}
			$\texttt{obj\_fun} \gets \texttt{generate\_obj\_fun}(\mu,\sigma,\gamma,\eta, X, Y)$\\
			\For{$j\in [p]$}{
				$(\mu_j, \sigma_j) \gets \mathtt{L\_BFGS}\big(\mathtt{obj\_fun.at}(j)\big)$
				\tcp*{minimizes eq.\eqref{CAVI_mu}}
			}
			$\gamma \gets \mathtt{update\_alpha}(\mu,\sigma,\gamma,\eta, X,Y)$
			\tcp*{implements eq.\eqref{CAVI_alpha}}
		}
		\KwOut{$\mu\in \R^{p}, \sigma^2\in\R_{>0}^{p}, \gamma \in [0,1]^p$}
		\caption{\texttt{Modified CAVI for variational Bayes with Laplace slabs}}
	\end{algorithm}

\section{Numerical study}\label{sec:num:anal}	

We empirically compare the performance of our VB method \textbf{VB (Lap)} based on prior \eqref{prior_beta} with parameters $a_0=b_0=\lambda=1$, to other state-of-the-art Bayesian variable selection methods in a simple simulation study. We implemented Algorithm \ref{algorithm} in C++ using the \texttt{Rcpp} interface and used the Armadillo linear algebra library and ensmallen optimization library, see \cite{bhardwaj2018,sanderson2016}. Note that the VB objective function is highly non-convex and so the local minimum returned by Algorithm \ref{algorithm} does not necessarily equal the global minimizer to the VB optimzation problem \eqref{VB}. For comparison, we first consider the usual VB approach \textbf{VB (Gauss)}, where the same variational family \eqref{VarFamily} is used, but the Laplace slabs are replaced by standard normal distributions in the prior \eqref{prior_beta} (e.g. \cite{huang2016,logsdon2010,ormerod2017,titsias2011,zhang2019}). We also compare our approach with the \textbf{varbvs} \cite{carbonetto2012}, \textbf{SkinnyGibbs} \cite{narisetty2019}, \textbf{BinaryEMVS} \cite{mcdermott2016}, \textbf{BhGLM} \cite{yi2019} and \textbf{rstanarm} \cite{goodrich2020} packages. We provide a brief description of these methods in Section \ref{sec:method_description} in the supplement. Note that the choice of hyperparameters can affect each of these methods and performance gains are possible from using good data-driven choices. For instance, our algorithm is sensitive to the choice of $\lambda$, see Section \ref{sec:hyperparameter} in the supplement.


Due to space constraints, we provide here only one test case and defer the remaining tests to Section \ref{sec:extra_experiments} of the supplement. We take $n=250$, $p=500$ and $X$ to be a standard Gaussian design matrix, i.e. $X_{ij}\sim^{iid} N(0,1)$, and set  the true signal $\theta_0 = (2,2,0,\dots,0)^T$ to be $s=2$ sparse. We ran the experiment 200 times for each method and report the means and standard deviations of the following performance measures: (i) true positive rate (TPR), (ii) false discovery rate (FDR), (iii) $\ell_2$-loss of the posterior mean $\hat{\theta}$, i.e. $\|\hat{\theta}-\theta_0\|_2$, (iv) mean-squared predictive error (MSPE) of the posterior mean, i.e. $(\tfrac{1}{n} \sum_{i=1}^n |\Psi(x_i^T\hat\theta)-\Psi(x_i^T\theta_0)|^2)^{1/2}$ and (v) run time in seconds. Since the BhGLM and rstanarm packages do not have explicit model selection subroutines, the TPR and FDR are not applicable. Similarly, SkinnyGibbs provides posterior model selection probabilities and not the posterior mean, hence neither the $\ell_2$-error or MSPE are applicable. The results are in Table \ref{table:sim1}.

\begin{threeparttable}
\caption{Comparing sparse Bayesian methods in high-dimensional logistic regression.}\label{table:sim1}
  \renewrobustcmd{\bfseries}{\fontseries{b}\selectfont}
  \begin{tabular}{lSSSSS}
    \toprule
    Algorithm & {TPR} & {FDR} & {$\ell_2$-error} & {MSPE} & {Time} \\
    \midrule
    VB (Lap) & \bfseries 1.00 \pm 0.00 & \bfseries 0.03 \pm 0.10 & \bfseries 0.57 \pm 0.37 & \bfseries 0.04 \pm 0.02 & 12.10 \pm 0.49 \\
    VB (Gauss) & \bfseries 1.00 \pm 0.00 & 0.84 \pm 0.04 & 3.34 \pm 0.46 & 0.25 \pm 0.03 & 0.98 \pm 0.51 \\
    varbvs & \bfseries 1.00 \pm 0.00 & 0.92 \pm 0.01 & 1.30 \pm 0.15 & 0.16 \pm 0.02 &\bfseries 0.08 \pm 0.01 \\
    SkinnyGibbs & \bfseries 1.00 \pm 0.00 & 0.90 \pm 0.01 & \textcolor{white}{00.00}\,\,\, \text{--} \textcolor{white}{00.00} &  \textcolor{white}{00.00}\,\,\, \text{--} \textcolor{white}{00.00} & 22.58 \pm 3.58 \\
    BinEMVS & \bfseries 1.00 \pm 0.00 & 0.88 \pm 0.03 & 4.25 \pm 0.56 & 0.29 \pm 0.02 & 31.04 \pm 0.94\\
    BhGLM & \textcolor{white}{00.00}\,\,\, \text{--} \textcolor{white}{00.00} & \textcolor{white}{00.00}\,\,\, \text{--} \textcolor{white}{00.00} & 2.69 \pm 0.88 & 0.21 \pm 0.16 & 1.60 \pm 0.37 \\
    rstanarm & \textcolor{white}{00.00}\,\,\, \text{--} \textcolor{white}{00.00} & \textcolor{white}{00.00}\,\,\, \text{--} \textcolor{white}{00.00} & 0.74 \pm 0.58 & 0.31 \pm 0.24 & 197.20 \pm 26.16 \\
    \bottomrule
  \end{tabular}
  \begin{tablenotes}
    \small
  \item Results based on 200 runs for an i.i.d. standard Gaussian design matrix $X\in\mathbb{R}^{250\times 500}$ with $X_{ij} \sim^{iid} N(0,1)$ and true signal $\theta_{0,1}=\theta_{0,2}=2$, $\theta_{0,j}=0$ for $3\leq j\leq 500$.
  \end{tablenotes}
\end{threeparttable}

In Table \ref{table:sim1} and the additional simulations in Section \ref{sec:suppl:sim}, we see that using Laplace slabs in the prior \eqref{prior_beta} generally outperforms the commonly used Gaussian slabs in all statistical metrics ($\ell_2$-loss, MPSE, FDR), in some cases substantially so.
This highlights the empirical advantages of using Laplace rather than Gaussian slabs for the prior underlying the VB approximation and matches the theory presented in Section \ref{sec:results}, as well as similar observations in linear regression \cite{rayszabo2019}. We also highlight the excellent FDR of VB (Lap), which warrants further investigation given its importance in Bayesian variable selection and multiple hypothesis testing.

However, the computational run-time is substantially slower for Laplace slabs due to the absence of analytically tractable update formulas as in the Gaussian case. The optimization routines required in Algorithm \ref{algorithm} mean a naive implementation can significantly increase the run-time; we are currently working on a more efficient implementation as an R-package \texttt{sparsevb} \cite{sparsevb} that should reduce the run-time by at least an order of magnitude.

The other methods perform roughly similarly to our algorithm in terms of estimation ($\ell_2$-error) and prediction (MSPE) error, doing better in certain test cases and worse in others. It seems there is no clearly dominant Bayesian approach regarding accuracy. However, our method provides the best results concerning model selection, generally having the best FDR while maintaining a competitive TPR, as suggested by Theorem \ref{thm:selection}. The other methods all perform comparably, with varbvs having the best TPR but substantially higher FDR, meaning it identifies many coefficients to be significant, both correctly and incorrectly.

We note that the two VB methods based on Gaussian prior slabs (VB (Gauss) and varbvs), which use analytic update formulas, are all significantly faster than the other methods. All the VB methods (including ours) scale much better with larger model sizes (e.g. $p=5000$) than the other methods, which did not finish running in a reasonable amount of time when $p=5000$, see Section \ref{sec:extra_experiments}. As expected, the MCMC methods generally performed slowest, though SkinnyGibbs, which is designed to scale up MCMC to larger sparse models, is indeed an order of magnitude faster than rstanarm.

\subsection{Coverage of marginal credible intervals}

An advantage of Bayesian methods is their ability to perform uncertainty quantification via credible sets. The present mean-field VB approximation provides access to marginal credible sets for individual features, which can be more informative than just the VB posterior mean, and are often of interest to practitioners. However, VB is known to generally underestimate the posterior variance, which can lead to bad uncertainty quantification. In view of the excellent FDR control of our VB method in earlier simulations, we further investigate the performance of these marginal credible sets empirically. 

We consider 4 tests cases, consisting of the above example (Test 0) and Tests 1-3 from Section \ref{sec:extra_experiments}. In each case, we computed 95\% marginal credible intervals for the coefficients, i.e. the intervals $I_j$, $j=1,\dots,p$, of smallest length such that $Q^*(\theta_j \in I_j) \geq 0.95$. We ran each experiments 200 times and report the mean and standard deviation of the coverage and length of these credible intervals for both the (true) zero and non-zero coefficients in Table \ref{tab:coverage}.

\begin{threeparttable}
  \caption{Marginal VB credible intervals for individual features}\label{tab:coverage}
  \renewrobustcmd{\bfseries}{\fontseries{b}\selectfont}
  \begin{tabular}{lSSSS}
    \toprule
    & {Test 0} & {Test 1} & {Test 2} & {Test 3} \\
    \midrule
    \textrm{Coverage (non-zero coefficients)} & 1.00 \pm 0.00 & 0.82 \pm 0.19 & 0.94 \pm 0.11 & 0.38 \pm 0.12 \\
    \textrm{Length (non-zero coefficients)} & 2.87 \pm 0.14 & 2.56 \pm 0.20 & 2.27 \pm 0.16 & 1.33 \pm 0.32 \\
    \textrm{Coverage (zero coefficients)} & 1.00 \pm 0.00 & 1.00 \pm 0.00 & 1.00 \pm 0.00 & 0.99 \pm 0.01 \\
    \textrm{Length (zero coefficients)} & 0.00 \pm 0.00 & 0.01 \pm 0.01 & 0.00 \pm 0.00 & 0.03 \pm 0.01 \\
    \bottomrule
  \end{tabular}
  \begin{tablenotes}
    \small
\item $X\in \mathbb{R}^{250\times 500}$, $X_{ij} \sim^{iid} N(0,\sigma^2)$. (0) $\sigma=1$, $s = 2$, $\theta_{0,1:s} = 2$; (1)  $\sigma=0.25$, $s=5$, $\theta_{0,1:s} = 4$; (2) $\sigma=2$, $s=10$, $\theta_{0,1:s} = 6$; (3) $\sigma=0.5$, $s=15$, $\theta_{0,1:s}  \sim^{iid} \mathrm{Unif}(-2,2)$.
  \end{tablenotes}
\end{threeparttable}

We see that for the true zero coefficients $\theta_{0,j} = 0$, the coverage is close to one with intervals of nearly zero width, meaning the VB posterior sets $\gamma_j = Q^* (\theta_j \neq 0) < 0.05$. This matches the other evidence (Theorem \ref{thm:selection} and the good FDR) that the VB posterior $Q^*$ does not include too many spurious variables. For true non-zero coefficients $\theta_{0,j} \neq 0$, the coverage is moderate to excellent in the first 3 experiments, which matches the good TPR seen above. However, coverage is low in Test 3, which is an especially difficult test typically containing several small non-zero coefficients (here $\theta_{0,j} \sim^{iid}\text{Unif}(-2,2)$) that are hard to detect. The low coverage is not surprising, since it is known that even in the full Bayes case, coefficients below a certain size cannot be consistently covered \cite{vanderpas2017}.

The very promising results here suggest our VB approach might be effective for uncertainty quantification for individual features, however this requires further investigation. We lastly note that while it may be reasonable to consider marginal credible intervals, one should be careful about using more general VB credible sets due to the mean-field approximation.

\section{Discussion}\label{sec:discussion}

This paper investigates a scalable and interpretable mean-field variational approximation of the popular spike and slab prior with Laplace slabs in high-dimensional logistic regression. We derive theoretical guarantees for this approach, proving (1) optimal concentration rates for the VB posterior in $\ell_2$ and prediction loss around a sparse true parameter, (2) optimal convergence rates for the VB posterior predictive mean and (3) that the VB posterior does not select overly large models, thereby controlling the number of false discoveries.

We verify in a numerical study that the empirical performance of the proposed method reflects these theoretical guarantees. In particular, using Laplace slabs in the prior underlying the variational approximation can substantially outperform the same VB method with Gaussian prior slabs, as is typically used in the literature, though at the expense of slower computation. The proposed approach performs comparably with other state-of-the-art sparse high-dimensional Bayesian variable selection methods for logistic regression, but scales substantially better to high-dimensional models where other approaches based on the EM algorithm or MCMC are not computable. We are currently working on a more efficient implementation as an R-package \texttt{sparsevb} \cite{sparsevb} that should improve the run-time.

Based on the promising FDR control and coverage of our VB method in the simulations, we plan to further investigate the theoretical and empirical performance of our algorithm for multiple hypothesis testing and variable selection, see \cite{castillo2018} for promising first results for the (unscalable) original posterior. Furthermore, the results derived here are the first steps towards better understanding VB methods in sparse high-dimensional nonlinear models. It opens up several interesting future lines of research for applying scalable VB implementations of spike and slab priors in complex high-dimensional models, including Bayesian neural networks \cite{polson:2018}, (causal) graphical models \cite{li:2018} and high-dimensional Bayesian time series \cite{scott:2013}.

\textbf{Acknowledgements}: We thank 4 reviewers for their useful comments that helped improve the presentation of this work.
Botond Szab\'o received funding from the Netherlands Organization for Scientific Research (NWO) under Project number: 639.031.654. 

\textbf{Impact statement}: Our theoretical results seek to better understand how sparse VB approximations work and thus improve their performance and reliability in practice. Since our results have no specific applications in mind, seeking rather to explain and improve an existing method, any potential broader impact will derive from improved performance in fields where such methods are already used.

\section*{Supplementary material for `Spike and slab variational Bayes for high dimensional logistic regression'}

In this supplement we present streamlined proofs for the asymptotic results (Section \ref{sec:proofs}), additional simulations (Section \ref{sec:suppl:sim}), discussion concerning the design matrix conditions (Section \ref{sec:extra_design}), full statements and proofs of the non-asymptotic results (Section \ref{sec:suppl:proofs}) and a derivation of the VB algorithm (Section \ref{sec:algorithm_derivation}).
 
\section{Proofs}\label{sec:proofs}

Our proofs use the following general result, which allows us to bound the VB probabilities of sets having exponentially small probability under the true posterior.

\begin{lemma}[Theorem 5 of \cite{rayszabo2019}]\label{lem:KL_VB_thm}
Let $\Theta_n$ be a subset of the parameter space, $A$ be an event and $Q$ be a distribution for $\theta$. If there exists $C>0$ and $\delta_n>0$ such that
\begin{equation}
E_{\theta_0} \Pi(\theta \in \Theta_n|Y)1_A \leq C e^{-\delta_n},\label{cond:KL_VB_thm}
\end{equation}
then
$$E_{\theta_0} Q(\theta\in \Theta_n)1_A \leq \tfrac{2}{\delta_n} \left[ E_{\theta_0} \KL (Q||\Pi(\cdot|Y))1_A + Ce^{-\delta_n/2} \right].$$
\end{lemma}
We must thus show that there exist events $(A_n)$ satisfying $P_{\theta_0}(A_n) \to 1$ such that on $A_n$:
\begin{enumerate}
\item the posterior puts at most exponentially small probability $Ce^{-\delta_n}$ outside $\Theta_n$,
\item the KL divergence between the VB posterior $Q^*$ and true posterior is $O(\delta_n)$.
\end{enumerate}
To aid readibility, in this section we state all intermediate results in asymptotic form, keeping track of only the leading order terms as $n,p \to \infty$. We provide full non-asymptotic statements in Section \ref{sec:suppl:proofs}, which may be skipped on first reading, though most of the technical difficulty is contained in these results.

For $t,L,M_1,M_2>0$, define the events
\begin{equation*}
\begin{split}
& \mA_{n,1} (t) = \{\|\nabla_\theta \ell_{n,\theta_0}(Y)\|_\infty \leq t\},\\
& \mA_{n,2} (L) = \{ \Pi( \theta \in \R^p: |S_\theta| \leq Ls_0 |Y) \geq 3/4\},\\
& \mA_{n,3} (M_1,M_2) = \{ \Pi(\theta\in \R^p: \|\theta - \theta_0\|_2 > M_1\sqrt{s_0\log p}/\|X\| |Y) \leq e^{-M_2s_0\log p} \}.
\end{split}
\end{equation*}
The first event bounds the score function $\nabla_\theta \ell_{n,\theta_0}(Y)$ and is needed to control the first order term in the Taylor expansion of the log-likelihood, see \eqref{Ln}. The second says the posterior concentrates on models of size at most a constant multiple of the true model dimension.
The last event says the posterior places all but exponentially small probability on an $\ell_2$-ball of the optimal radius about the truth and is used for a localization argument.

The event required to apply Lemma \ref{lem:KL_VB_thm} is defined by
\begin{equation}\label{An event}
\mA_n (t,L,M_1,M_2) = \mA_{n,1}(t) \cap \mA_{n,2}(L) \cap \mA_{n,3}(M_1,M_2).
\end{equation}
The proof has an iterative structure, where we localize the posterior based on the event $\mA_{n,i-1}$ to prove $\mA_{n,i}$ has high probability. The idea to iteratively localize the posterior during the proofs is a useful technique from Bayesian nonparametrics (e.g. \cite{nicklray2019}).

\begin{lemma}\label{lem:prob:asymp}
Suppose the prior satisfies \eqref{prior_cond} and \eqref{lambda}, and the design matrix satisfies condition \eqref{cond:design} for  $L=2\max\{ A_4/5,(1.1 + 4\alpha^2/\underline{\kappa}+2A_4 +\log(4+\upk (s_0)))/A_4\}$. Then for $p$ large enough,
\begin{align*}
P_{\theta_0}\Big( \mA_n (t,L,M_1,M_2)^c  \Big)= O(1/p),
\end{align*}
with $t=\|X\|\sqrt{\log p}$, $M_1=D_0\sqrt{L}/\downk ((1+4L/A_4)s_0)$, $D_0 =(24/\sqrt{A_4})\max(25\alpha,\frac{1+A_3}{16},\frac{3A_4+4}{32})$, and $M_2=2L$.
\end{lemma}

The next two lemmas establish exponential posterior bounds on the event $\mA_{n,1}( \|X\|\sqrt{\log p}) \supset \mA_n(t,L,M_1,M_2)$. The first states that the posterior concentrates on models of size at most a constant multiple of the true model size $s_0=|S_{\theta_0}|$. The second states that the posterior concentrates on an $\ell_2$-ball of optimal radius about the true parameter $\theta_0$.

\begin{lemma}\label{lem: thm4(1):asymp}
Suppose the prior satisfies \eqref{prior_cond} and \eqref{lambda}. If for $L\geq 2(1.1 + 4\alpha^2/\underline{\kappa}+2A_4 +\log(4+\upk (s_0)))/A_4$ the design matrix satisfies condition \eqref{cond:design}, then for $p$ large enough,
\begin{align*}
E_{\theta_0}[\Pi \left( \theta\in \R^p: |S_\theta| \geq L s_0 \mid Y \right) 1_{\mA_{n,1}(\|X\|\sqrt{\log p}) }]   \leq 2 \exp \left\{- (LA_4/2)s_0 \log p  \right\}.
\end{align*}
\end{lemma}

\begin{lemma}\label{lem: combined:contraction:asymp}
Suppose the prior satisfies \eqref{prior_cond} and \eqref{lambda}. If for $K>\max\{ A_4,2(1.1+4\alpha^2/\downk + 2A_4 + \log(4+\upk(s_0))/A_4\}$, the design matrix satisfies \eqref{cond:design} with $L=2K/A_4$, then for $p$ large enough,
\begin{align*}
E_{\theta_0} \left[\Pi \left( \theta \in \R^p: \|\theta-\theta_0\|_2 \geq  \frac{D_1\sqrt{K}}{\underline\kappa\big((2K/A_4+1)s_0\big)} \frac{\sqrt{s_0\log p}}{\|X\| } \bigg|Y \right)1_{\mA_{n,1}( \|X\|\sqrt{\log p})} \right] \leq 8  e^{ -Ks_0\log p },
\end{align*}
where  $D_1=16 A_4^{-1/2} \max\{25\alpha,\frac{3+2A_2+A_3}{16}\}$.
\end{lemma}

Lemmas \ref{lem:prob:asymp}, \ref{lem: thm4(1):asymp} and \ref{lem: combined:contraction:asymp} follow immediately from their non-asymptotic counterparts Lemmas \ref{lem:prob}, \ref{lem: thm4(1)} and \ref{lem: combined:contraction}, respectively, in Section \ref{sec:suppl:proofs}. We finally control the KL divergence between the VB posterior $Q^*$ and posterior $\Pi(\cdot|Y)$ on the event $\mA_n$; see Lemma \ref{lem:KL} for the corresponding non-asymptotic result. This is the most difficult technical step in establishing our result.

\begin{lemma}\label{lem:KL:asymp}
 Consider the event $\mA_n = \mA_n(t,L,M_1,M_2)$ in Lemma \ref{lem:prob:asymp}. Then for sufficiently large $p$, the VB posterior $Q^*$ satisfies 
\begin{align*}
\KL (Q^*||\Pi(\cdot|Y))1_{\mA_n} \leq D_2 s_0\log p ,
\end{align*}
with $D_2 = L(\frac{(9D_0/4+ \alpha/2)D_0\bar\kappa(Ls_0)}{\downk((1+4L/A_4)s_0)^2} + \frac{\alpha}{2\downk (Ls_0)})$ and where $L,D_0$ are given in Lemma \ref{lem:prob:asymp}.
\end{lemma}


We briefly explain the heuristic idea behind the proof of Lemma \ref{lem:KL:asymp}. Since the VB posterior $Q^*$ is the minimizer of the KL objective \eqref{VB}, $\KL (Q^*||\Pi(\cdot|Y))\leq \KL (Q||\Pi(\cdot|Y))$ for any $Q \in \Q$ in the variational family. We upper bound this quantity for some $Q$ carefully chosen according to the logistic likelihood. We first identify a model $S \subseteq \{1,\dots,p\}$ which is not too far from the true model $S_{\theta_0}$ and to which the posterior assigns sufficient, though potentially exponentially small, probability. In the low dimensional setting $p\ll n$, Taylor expanding the log-likelihood $\ell_{n,\theta}-\ell_{n,\theta_0}$ asymptotically gives a Gaussian linear regression likelihood with rescaled design matrix and data. This motivates the distribution $Q\in \Q$ which fits a normal distribution $N_S(\mu_S,D_S)$ on $S$ with mean $\mu_S$ equal to the least squares estimator solving the linearized Gaussian approximation and covariance $D_S$ (a diagonalized version of) the covariance matrix of this estimator, again in the linearized model. While the Taylor expansion is not actually valid in the sparse high-dimensional setting $p\gg n$ considered here, we can nonetheless show that the approximation is still sufficiently good to apply Lemma \ref{lem:KL_VB_thm}.

\begin{proof}[Proof of Theorem \ref{thm:contraction}]
We apply Lemma \ref{lem:KL_VB_thm} with 
\begin{align*}
\Theta_n=\Big\{\theta\in\mathbb{R}^p:\, \|\theta-\theta_0\|_2\geq \frac{M_n^{1/2}\sqrt{s_0\log p}}{\underline\kappa(M_ns_0)\|X\| }\Big\}
\end{align*}
and $A = \mA_n = \mA_n(t,L,M_1,M_2)$ the event in Lemma \ref{lem:prob:asymp}. Since $\mA_n \subset \mA_{n,1}(\|X\|\sqrt{\log p})$, Lemma \ref{lem: combined:contraction:asymp} implies that $E_{\theta_0} \Pi(\theta \in \Theta_n|Y)1_{\mA_n} \leq 8 e^{-D_1^{-2}M_n s_0\log p}$, i.e. \eqref{cond:KL_VB_thm} holds with $\delta_n=D_1^{-2}M_n s_0\log p$. Using Lemma \ref{lem:KL_VB_thm} followed by Lemma \ref{lem:KL:asymp}, since $\delta_n \to \infty$,
\begin{align*}
E_{\theta_0} Q^*(\theta\in \Theta_n)1_{\mA_n} &\leq \tfrac{2}{\delta_n}\left[ E_{\theta_0} \KL (Q^*||\Pi(\cdot|Y))1_{\mA_n} + 8e^{-\delta_n/2} \right] \leq \frac{2\alpha^2 D_2}{M_n} (1+o(1)) = O(C_\kappa/M_n).
\end{align*}
Since $P_{\theta_0}(\mA_n) \to 1$ by Lemma \ref{lem:prob:asymp}, the first result follows. Since $\|\Psi'\|_\infty \leq 1/4$,
\begin{equation}
n\|p_\theta-p_{0}\|_n^2  \leq \tfrac{1}{16}\sum_{i=1}^n |x_i^T(\theta-\theta_0)|^2 = \tfrac{1}{16}\|X(\theta-\theta_0)\|_2^2 \leq \tfrac{1}{16}\overline{\kappa}((M_n+1)s_0) \|X\|^2 \|\theta-\theta_0\|_2^2,\label{eq:help:UB:pred}
\end{equation}
where the last inequality follows from Theorem \ref{thm:selection} and the definition of $\overline{\kappa}(\cdot)$. The second statement then follows by combining the first statement of the theorem with the above display.
\end{proof}

\begin{proof}[Proof of Theorem \ref{thm:convergence:postmean}]
By the duality formula for the KL divergence (\cite{boucheron:2013}, Corollary 4.15),
\begin{align*}
\int f(\theta) dQ^*(\theta)\leq \KL(Q^*\| \Pi(\cdot|Y)  )+\log \int e^{f(\theta)}d\Pi(\theta|Y)
\end{align*}
for any measurable $f$ such that $\int e^{f(\theta)} d\Pi(\theta|Y)<\infty$. Let $\mA_n = \mA(t,L,M_1,M_2)$ be the event in Lemma \ref{lem:prob:asymp}. Applying Jensen's inequality and taking $f(\theta)=cn\|p_\theta-p_0\|_n^2$ for $c>0$,
\begin{align}
cn \|\hat{p}^*-p_0\|_n^21_{\mathcal{A}_n}&= cn\|E^{Q^*}p_\theta-p_0\|_n^21_{\mathcal{A}_n}\leq E^{Q^*}cn\|p_\theta-p_0\|_n^2 1_{\mathcal{A}_n}\nonumber\\
&  \leq  \KL(Q^*\| \Pi(\cdot|Y)  )1_{\mathcal{A}_n} +1_{\mA_n} \log \int e^{cn\|p_\theta-p_0\|_n^2}d\Pi(\theta|Y).\label{eq:help:postmean}
\end{align}
By Lemma \ref{lem:KL:asymp}, the first term is bounded by $D_2s_0 \log p$. For notational convenience, write $\overline{\kappa}^*=\overline{\kappa}( \frac{2n}{A_4\log p}+ s_0)$ and $\underline{\kappa}^*=\underline{\kappa}( \frac{2n}{A_4\log p} +(1-2/A_4)s_0)$. For $D_1$ defined in Lemma \ref{lem: combined:contraction:asymp} and $K>0$ the minimal constant satisfying the conditions of Lemma \ref{lem: combined:contraction:asymp}, set
\begin{align*}
\mathcal{B}_0&= \left\{\theta\in\mathbb{R}^p:\,|S_\theta|\leq \tfrac{2n}{A_4\log p}-1,\,  n\|p_\theta-p_0\|_n^2\leq K \tfrac{D_1^2 \bar{\kappa}^* s_0\log p }{(\underline{\kappa}^*)^2}\right\},  \\
\mathcal{B}_j&=\Big\{\theta\in\mathbb{R}^p:\, |S_\theta|\leq \tfrac{2n}{A_4\log p}-1,\,  j\tfrac{ D_1^2\bar{\kappa}^* s_0\log p }{(\underline{\kappa}^*)^2}< n\|p_\theta-p_0\|_n^2\leq(j+1) \tfrac{D_1^2\bar{\kappa}^* s_0\log p }{(\underline{\kappa}^*)^2} \Big\},\\
\bar{\mathcal{B}}&= \left\{\theta\in\mathbb{R}^p:\, |S_\theta|> \tfrac{2n}{A_4\log p}-1\right\}.
\end{align*}
Since $E[U1_\Omega] = E[U|\Omega]P(\Omega)$, conditional Jensen's inequality gives 
$$E[(\log V) 1_\Omega] \leq P(\Omega) \log E[V|\Omega]= P(\Omega) \log E[V1_\Omega] - P(\Omega) \log P(\Omega)$$
for any random variable $V$ and event $\Omega$. The $E_{\theta_0}$-expectation of the second term in \eqref{eq:help:postmean} thus equals
\begin{align*}
& E_{\theta_0} \bigg[ \log\bigg(  \int_{\mathcal{B}_0}   e^{cn\|p_\theta-p_0\|_n^2}d\Pi(\theta|Y) + \sum_{j=K}^{n/(s_0\log p)-1}\int_{\mathcal{B}_j}   e^{cn\|p_\theta-p_0\|_n^2}d\Pi(\theta|Y)  \\
& \qquad \qquad \qquad \qquad \qquad +\int_{\bar{\mathcal{B}}}   e^{cn\|p_\theta-p_0\|_n^2}d\Pi(\theta|Y) \bigg)1_{\mA_n} \bigg] \\
&\leq P_{\theta_0}(\mA_n) \log E_{\theta_0} \bigg[ e^{cK \tfrac{D_1^2 \bar{\kappa}^* s_0\log p }{(\underline{\kappa}^*)^2}} + 
\sum_{j=K}^{n/(s_0\log p)-1} e^{ c(j+1) \tfrac{ D_1^2 \bar{\kappa}^* s_0\log p }{(\underline{\kappa}^*)^2}}\Pi( \mathcal{B}_j|Y)1_{\mA_n} + e^{cn} \Pi( \bar{\mathcal{B}}|Y) 1_{\mA_n} \bigg] \\
& \qquad - P_{\theta_0}(\mA_n) \log P_{\theta_0}(\mA_n).
\end{align*}
Using \eqref{eq:help:UB:pred} and Lemma \ref{lem: combined:contraction:asymp} gives $E_{\theta_0}[\Pi( \mathcal{B}_j|Y)1_{\mA_n}] \leq 8e^{-js_0\log p}$, while Lemma \ref{lem: thm4(1):asymp} gives $E_{\theta_0}[\Pi( \bar{\mathcal{B}}|Y) 1_{\mA_n}] \leq 2e^{-n}$. Taking  $c= c_n = \frac{(\underline{\kappa}^*)^2}{2D_1^2\bar\kappa^*} \leq 1/(128)$ and using that $P_{\theta_0}(\mA_n^c) = O(1/p)$ by Lemma \ref{lem:prob:asymp}, the last display is bounded by,
\begin{align*}
& \log \bigg[ e^{(K/2) s_0\log p } + 8\sum_{j=K}^{n/(s_0\log p)-1} e^{ (\frac{j+1}{2}-j)  s_0\log p} + 2e^{-(1-c)n} \bigg] +O(1/p) \leq Ks_0\log p(1+o(1)).
\end{align*}
Plugging in the preceding bounds for the right hand side of \eqref{eq:help:postmean},
\begin{align*}
n E_{\theta_0} \|\hat{p}^*-p_0\|_n^21_{\mathcal{A}_n}\leq c_n^{-1}(K+D_2)s_0\log p (1+o(1)) = O(c_n^{-1} C_\kappa s_0 \log p),
\end{align*}
for $C_\kappa$ the constant in Theorem \ref{thm:contraction}. Applying Markov's inequality and Lemma \ref{lem:prob:asymp},
$$ P_{\theta_0}(\|\hat{p}^*-p_0\|_n \geq r) \leq r^{-2} E_{\theta_0} [\|\hat{p}^*-p_0\|_n^21_{\mathcal{A}_n}] + P_{\theta_0}(\mA_n^c)  = O(r^{-2}n^{-1} c_n^{-1} C_\kappa s_0 \log p) + o(1).$$
Taking $r=M_n^{1/2} (\upk^*)^{1/2} (\downk^*)^{-1} \sqrt{s_0\log p/n}$ gives the result.
\end{proof}

\begin{proof}[Proof of Theorem \ref{thm:selection}]
This follows the same argument as the proof of Theorem \ref{thm:contraction}, taking $\Theta_n=\{\theta\in\mathbb{R}^p:\, |S_\theta| \geq M_n s_0\}$, $\delta_n=(A_4/2)M_n s_0\log p$ and applying Lemma \ref{lem: thm4(1):asymp} instead of Lemma \ref{lem: combined:contraction:asymp}.
\end{proof}

\section{Additional numerical results}\label{sec:suppl:sim}

\subsection{Description of Bayesian methods for logistic regression}\label{sec:method_description}
In the \textbf{varbvs} package \cite{carbonetto2012}, the standard VB algorithm was implemented for the prior \eqref{prior} with Gaussian instead of Laplace slabs and an additional layer of importance sampling for computing the low-dimensional prior hyperparameters. As for our algorithm, we set $a_0=b_0=1$. In the \textbf{BinaryEMVS} package \cite{mcdermott2016}, an expectation-maximization (EM) algorithm is implemented for fitting Bayesian spike and slab regularization paths for logistic regression. More concretely, the considered spike and slab prior takes the form 
\begin{align*}
\omega&\sim \text{Beta}(a_0,b_0),\\
z_j&\sim^{iid}\text{Bern}(\omega),\\
\theta_j &\sim^{iid} (1-z_j) \mathcal{N}(0, \sigma_1^2) + z_j\mathcal{N}(0, \sigma_2^2),\qquad \text{with $\sigma_1 \ll \sigma_2$}.
\end{align*}
In the simulation study, we use the parametrization $a_0=1$, $b_0=1$, $\sigma_1=0.025$ and $\sigma_2=5$. In the Bayesian hierarchical generalized linear model (\textbf{BhGLM} package) of \cite{yi2019}, an EM algorithm is implemented for a mixture of Laplace priors:
\begin{align*}
\omega&\sim \text{Unif}[0,1],\\
z_j&\sim^{iid}\text{Bern}(\omega),\\
\theta_j &\sim^{iid} (1-z_j)\text{Lap}(\lambda_1) + z_j \text{Lap}(\lambda_2) ,\qquad \text{with $\lambda_1\gg \lambda_2$}.
\end{align*}
We also work with the default parametrization of the bmlasso function of the BhGLM package, i.e. $s_1=1/\lambda_1=0.04$ and $s_2=1/\lambda_2=0.5$. We use the implementation of \textbf{SkinnyGibbs} in the supplementary material to \cite{narisetty2019}, using the function skinnybasad with the settings provided in the example from the manual, apart from taking pr=0.5 to reflect the present prior setting $a_0=b_0=1$.
Finally, we consider the \textbf{rstanarm} package, which makes Bayesian regression modeling via the probabilistic programming language Stan accessible in R. In the package, Hamiltonian Monte Carlo \cite{hoffman2011} was implemented for the horseshoe global-local shrinkage prior \cite{carvalho2010}. We run the function stan\_glm, again with the default parameterization, i.e. we set the global scale to 0.01 with degree of freedom 1, and the slab-scale to 2.5 with degrees of freedom 4. The default inferential algorithm runs 4 Markov chains with 2000 iterations each.

\subsection{Additional experiments}\label{sec:extra_experiments}
We provide five further test cases in addition to the experiment considered in Section \ref{sec:num:anal}. In all cases we consider Gaussian design matrices, but vary all other parameter. In tests 1-3, we take $n=250$ and $p=500$ as in Section  \ref{sec:num:anal}, while in experiments 4 and 5 we set $n=2500$ and $p=5000$. The entries of the design matrices have independent centered Gaussian distributions with standard deviations $\sigma=0.25,\ 2,\ 0.5,\ 0.5,\ 1$, respectively. The true underlying signal has sparsity levels $s=5,\,10,\, 15,\,25,\ 25$, respectively, with the non-zero signal coefficients located at the beginning of the signal with values equal to (1) $\theta_{0,j}=4$, (2) $\theta_{0,j} = 6$, (3) $\theta_{0,j} \sim^{iid}\text{Unif}(-2,2)$, (4) $\theta_{0,j} = 2$ and (5) $\theta_{0,j} \sim^{iid}\text{Unif}(-1,1)$ for $j=1,\dots,s$.

We ran each experiment 200 times and report the means and standard deviations of the performance metrics in Table \ref{tab::test_results}.
The $\ell_2$-error ($\ell_2(\hat\theta,\theta_0)=\|\hat\theta-\theta_0\|_2$) and mean squared prediction error ($\text{MSPE}(\hat{p}) = (\tfrac{1}{n} \sum_{i=1}^n|\Psi(x_i^T\hat{\theta})- \Psi(x_i^T\theta_0)|^2)^{1/2}$) are reported with respect to the posterior mean. For the methods performing model selection, we use the standard threshold 0.5 for the marginal posterior inclusion probability $\alpha_j$, i.e. the posterior includes the $j$th coefficient in the model if $\alpha_j>0.5$. The true positive rate (TPR) and false discovery rate (FDR) are then defined as $\text{TPR}(\alpha)=s^{-1}\sum_{j:\, \theta_{0,j}\neq 0}1_{\alpha_j>0.5}$ and $\text{FDR}(\alpha)=\sum_{j:\, \alpha_j>0.5}1_{\theta_{0,j}=0}/|\{j:\, \alpha_j>0.5\}|$, respectively. The elapsed times are given for an Intel \texttt{i7-8550u} laptop processor.

As in Section \ref{sec:num:anal}, we conclude that our VB approach typically outperforms the other variational algorithms using prior Gaussian slabs, though again at the expense of greater computational times. Compared to our approach, the other four methods based on the EM algorithm or MCMC performed better in some scenarios and worse in others, but with substantially greater computational times. Unsurprisingly, the rstanarm package using MCMC is the slowest; in many cases it did not even converge after 8000 iterations. In the high-dimensional case $p=5000$ and $n=2500$, 4 algorithms (SkinnyGibbs, BhGLM, BinEMVS, rstanarm) did not finish the computations in a reasonable amount of time: the fastest required at least 100 hours to execute 200 runs and for rstanarm, even a single run required multiple hours.

\begin{threeparttable}
  \caption{Comparing sparse Bayesian methods in high-dimensional logistic regression.}\label{tab::test_results}
  \renewrobustcmd{\bfseries}{\fontseries{b}\selectfont}
  \begin{tabular}{llSSSSS}
    \toprule
    & Algorithm & {Test 1} & {Test 2} & {Test 3} & {Test 4} & {Test 5} \\
    \midrule
    \multirow{4}{*}{TPR}& VB (Lap) & 0.99 \pm  0.06 & \bfseries 1.00 \pm 0.00 & 0.51 \pm 0.11 & \bfseries 1.00 \pm 0.00 & 0.40 \pm 0.28 \\
    & VB (Gauss) & 1.00 \pm 0.01 & 1.00 \pm  0.02 & 0.54 \pm  0.11 & \bfseries 1.00 \pm 0.00 & 0.85 \pm 0.06 \\
    & varbvs & \bfseries 1.00 \pm 0.00 & \bfseries 1.00 \pm 0.00 & \bfseries 0.68 \pm 0.11 & \bfseries 1.00 \pm 0.00 & \bfseries 0.87 \pm 0.06 \\
    & SkinnyGibbs & 0.98 \pm 0.06 & 1.00 \pm 0.02 & 0.51 \pm 0.12 & \textcolor{white}{00.00}\,\,\, \text{--} \textcolor{white}{00.00} & \textcolor{white}{00.00}\,\,\, \text{--} \textcolor{white}{00.00}\\
    & BinEMVS & 0.99 \pm 0.03 & \bfseries 1.00 \pm 0.00 & 0.58 \pm 0.11 & \textcolor{white}{00.00}\,\,\, \text{--} \textcolor{white}{00.00} & \textcolor{white}{00.00}\,\,\, \text{--} \textcolor{white}{00.00}\\
    \midrule
    \multirow{4}{*}{FDR}& VB (Lap) & 0.49 \pm 0.11 & \bfseries 0.00 \pm 0.02 & \bfseries 0.41 \pm 0.14 & \bfseries 0.01 \pm 0.02 & \bfseries 0.03 \pm 0.05 \\
    & VB (Gauss) & 0.63 \pm 0.07 & 0.09 \pm 0.13 & 0.52 \pm 0.12 & 0.81 \pm 0.02 & 0.95 \pm 0.01 \\
    & varbvs & 0.93 \pm 0.01 & 0.08 \pm 0.08 & 0.83 \pm 0.03 & 0.93 \pm 0.00 & 0.91 \pm 0.01 \\
    & SkinnyGibbs & 0.80 \pm 0.03 & 0.11 \pm 0.11 & 0.71 \pm 0.07 & \textcolor{white}{00.00}\,\,\, \text{--} \textcolor{white}{00.00} & \textcolor{white}{00.00}\,\,\, \text{--} \textcolor{white}{00.00}\\
    & BinEMVS & \bfseries 0.43 \pm 0.14 & 0.19 \pm 0.10 & 0.63 \pm 0.10 & \textcolor{white}{00.00}\,\,\, \text{--} \textcolor{white}{00.00} & \textcolor{white}{00.00}\,\,\, \text{--} \textcolor{white}{00.00}\\
    \midrule
    \multirow{6}{*}{$\ell_2$-Error}& VB (Lap) & 3.97 \pm 0.85 & \bfseries 1.73 \pm 0.59 & 4.89 \pm 1.29 & \bfseries 4.23 \pm 0.72 & 2.31 \pm 0.64 \\
    & VB (Gauss) & 4.99 \pm 0.46 & 13.82 \pm 0.16 & 3.86 \pm 0.61 & 8.34 \pm 0.49 & 17.05 \pm 0.52 \\
    & varbvs & 7.29 \pm 0.24 & 16.31 \pm 0.06 & 3.35 \pm 0.44 & 7.68 \pm 0.03 & \bfseries 1.32 \pm 0.13 \\
    & BhGLM & 4.39 \pm 0.65 & 15.68 \pm 0.53 & \bfseries 2.81 \pm 0.46 &  \textcolor{white}{00.00}\,\,\, \text{--} \textcolor{white}{00.00}& \textcolor{white}{00.00}\,\,\, \text{--} \textcolor{white}{00.00}\\
    & BinEMVS & 3.84 \pm 0.97 & 14.49 \pm 0.28 & 5.82 \pm 0.89 &  \textcolor{white}{00.00}\,\,\, \text{--} \textcolor{white}{00.00}& \textcolor{white}{00.00}\,\,\, \text{--} \textcolor{white}{00.00}\\
    & rstanarm & \bfseries 2.87 \pm 1.49 & 6.74 \pm 0.86 & 3.14 \pm 0.88 & \textcolor{white}{00.00}\,\,\, \text{--} \textcolor{white}{00.00} & \textcolor{white}{00.00}\,\,\, \text{--} \textcolor{white}{00.00}\\
    \midrule
    \multirow{6}{*}{MSPE}& VB (Lap) & 0.18 \pm 0.02 & 0.07 \pm 0.02 & 0.24 \pm 0.03 & \bfseries 0.06 \pm 0.01 & 0.22 \pm 0.09 \\
    & VB (Gauss) & 0.21 \pm 0.02 & 0.08 \pm 0.02 & 0.25 \pm 0.03 & 0.21 \pm 0.01 & 0.34 \pm 0.01 \\
    & varbvs & 0.21 \pm 0.01 & 0.12 \pm 0.01 & \bfseries 0.19 \pm 0.02 & 0.18 \pm 0.00 & \bfseries 0.17 \pm 0.01 \\
    & BhGLM & 0.22 \pm 0.15 & \bfseries 0.06 \pm 0.05 & 0.20 \pm 0.16 & \textcolor{white}{00.00}\,\,\, \text{--} \textcolor{white}{00.00}& \textcolor{white}{00.00}\,\,\, \text{--} \textcolor{white}{00.00}\\
    & BinEMVS & \bfseries 0.15 \pm 0.03 & 0.09 \pm 0.02 & 0.27 \pm 0.03 &  \textcolor{white}{00.00}\,\,\, \text{--} \textcolor{white}{00.00}& \textcolor{white}{00.00}\,\,\, \text{--} \textcolor{white}{00.00}\\ 
    & rstanarm &  0.36 \pm 0.25 & 0.10 \pm 0.08 & 0.37 \pm 0.30 & \textcolor{white}{00.00}\,\,\, \text{--} \textcolor{white}{00.00} & \textcolor{white}{00.00}\,\,\, \text{--} \textcolor{white}{00.00}\\
    \midrule
    \multirow{6}{*}{Time}& VB (Lap) & 8.66 \pm  6.19 & 12.05 \pm 0.55 & 14.53 \pm 0.78 & 360.22 \pm 10.97 & 358.79 \pm 8.43 \\
    & VB (Gauss) & 0.23 \pm 0.05 & 1.45 \pm  0.66 & 1.70 \pm  1.37 & 356.76 \pm 7.22 & 359.78 \pm 1.02 \\
    & varbvs & \bfseries 0.05 \pm 0.00 & \bfseries 0.10 \pm 0.03 & \bfseries 0.03 \pm 0.01 & \bfseries 4.29 \pm 0.12 & \bfseries 8.15 \pm 0.48 \\
    & SkinnyGibbs & 19.89 \pm 0.04 & 19.77 \pm 0.09 & 20.30 \pm 0.42 & \textcolor{white}{00.00}\,\,\, \text{--} \textcolor{white}{00.00} & \textcolor{white}{00.00}\,\,\, \text{--} \textcolor{white}{00.00}\\
    & BhGLM &  1.37 \pm 0.28 & 2.24 \pm 0.71 & 1.42 \pm 0.52 & \textcolor{white}{00.00}\,\,\, \text{--} \textcolor{white}{00.00} & \textcolor{white}{00.00}\,\,\, \text{--} \textcolor{white}{00.00}\\
    & BinEMVS & 12.43 \pm 4.28 & 36.52 \pm 4.82 & 30.44 \pm 2.88 & \textcolor{white}{00.00}\,\,\, \text{--} \textcolor{white}{00.00} & \textcolor{white}{00.00}\,\,\, \text{--} \textcolor{white}{00.00}\\
    & rstanarm &  152.46 \pm 18.00 & 426.20 \pm 50.67 & 181.53 \pm 26.05 & \textcolor{white}{00.00}\,\,\, \text{--} \textcolor{white}{00.00} & \textcolor{white}{00.00}\,\,\, \text{--} \textcolor{white}{00.00}\\
    \bottomrule
  \end{tabular}
  \begin{tablenotes}
    \small \item The design matrices $X\in\mathbb{R}^{n\times p}$ are taken to be $X_{ij}\sim^{iid}N(0,\sigma^2)$. The signal vector $\theta_0$ has $s$ non-zero coefficients, all located at the beginning of the signal.
  \item (1) $X\in\R^{250\times 500}$, $\sigma = 0.25$, $s = 5$, $\theta_{0,1:s} = 4$
  \item (2) $X\in\R^{250\times 500}$, $\sigma=2$, $s = 10$, $\theta_{0,1:s} = 6$
  \item (3) $X\in\R^{250\times 500}$, $\sigma=0.5$, $s = 15$, $\theta_{0,1:s} \sim^{iid} \mathrm{Unif}(-2,2)$
  \item (4) $X\in\R^{2500\times 5000}$, $\sigma=0.5$, $s = 25$, $\theta_{0,1:s} = 2$
  \item (5) $X\in\R^{2500\times 5000}$, $\sigma=1$, $s = 10$, $\theta_{0,1:s} \sim^{iid} \mathrm{Unif}(-1,1)$
  \end{tablenotes}
\end{threeparttable}

\subsection{Comparing different choices of the hyperparameter $\lambda$}\label{sec:hyperparameter}

Theorems \ref{thm:contraction}-\ref{thm:selection} state that for the hyperparameter choice $\lambda \asymp \|X\| \sqrt{\log p}$, our VB algorithm has good asymptotic properties. In practice, however, the finite-sample performance indeed depends on $\lambda$. We ran our algorithm 200 times on the experiment considered in Section \ref{sec:num:anal} for different choices of $\lambda$ and report the results in Table \ref{tab:hyperparameter}. In this example, the performance was sensitive to the choice of $\lambda$ with large values of $\lambda$, which cause more shrinkage, performing worse.

In linear regression, where more extensive simulations have been carried out, the choice of $\lambda$ was similarly found to have an effect, though there was not clear evidence to support a particular fixed choice of $\lambda$, with larger values sometimes performing better and sometime worse \cite{rayszabo2019}. This suggests using a data-driven choice of $\lambda$ may be helpful in practice, for example using cross validation.

\begin{threeparttable}
  \caption{Varying the scale hyperparameter}\label{tab:hyperparameter}
  \renewrobustcmd{\bfseries}{\fontseries{b}\selectfont}
  \begin{tabular}{lSSSSS}
    \toprule
    & {$\lambda = \frac{1}{20}$} & {$\lambda = \frac{1}{5}$} & {$\lambda = 2$} & {$\lambda = 5$} & {$\lambda = 20$} \\
    \midrule
    TPR  & 1.00 \pm 0.00 & 1.00 \pm 0.00 & 1.00 \pm 0.00 & 1.00 \pm 0.00 & 0.81 \pm 0.28 \\
    FDR  & 0.00 \pm 0.00 & 0.02 \pm 0.08 & 0.03 \pm 0.10 & 0.09 \pm 0.16 & 0.02 \pm 0.10 \\
    $\ell_2$-error & 0.53 \pm 0.36 & 0.58 \pm 0.37 & 0.48 \pm 0.28 & 0.39 \pm 0.14 & 1.73 \pm 0.44 \\
    MSPE & 0.04 \pm 0.02 & 0.04 \pm 0.02 & 0.04 \pm 0.02 & 0.04 \pm 0.01 & 0.17 \pm 0.07 \\
    Time & 12.13 \pm 0.58 & 12.02 \pm 0.73 & 11.93 \pm 0.71 & 11.72 \pm 1.05 & 12.10 \pm 0.73 \\
    \bottomrule
  \end{tabular}
  \begin{tablenotes}
    \small
  \item $X\in\R^{250\times 500}$, $X_{ij} \sim^{iid} N(0,1)$, $s = 2$, $\theta_{0,1:s} = 2$
  \end{tablenotes}
\end{threeparttable}

\section{Further discussion of the design matrix and sparsity assumptions}\label{sec:extra_design}

For $s\in \{1,\dots,p\}$, recall the definitions
$$\underline{\kappa} = \inf \left\{ \frac{\|W^{1/2}X\theta\|_2^2}{\|X\|^2 \|\theta\|_2^2} : \|\theta_{S_0^c}\|_1 \leq 7\|\theta_{S_0}\|_1, \theta \neq 0 \right\},$$
$$\overline{\kappa}(s) = \sup \left\{ \frac{\|X\theta\|_2^2}{\|X\|^2 \|\theta\|_2^2}: 0 \neq |S_\theta| \leq s \right\}, \qquad \underline{\kappa}(s) = \inf \left\{ \frac{\|W^{1/2}X\theta\|_2^2}{\|X\|^2 \|\theta\|_2^2}: 0 \neq |S_\theta| \leq s \right\}.$$
These definitions are taken from the sparsity literature \cite{buhlmann2011} and are used in sparse logistic regression \cite{atchade2017,negahban2012,wei2020}. We reproduce some of the discussion here for convenience, but refer the reader to Chapter 6 of \cite{buhlmann2011} for further reading.

The true model $S_0$ is \textit{compatible} if $\downk >0$, which implies $\|W^{1/2}X\theta\|_2^2 \geq \downk \|X\|^2 \|\theta\|_2^2$ for all $\theta$ in the relevant set. The number 7 can be altered and is taken to match the conditions used in \cite{atchade2017,castillo2015} since we use some of their results. Compatibility $\downk >0$ involves approximate rather than exact sparsity, since the parameters $\theta$ need only have small rather than zero coordinates outside $S_0$. In contrast, $\downk(s)$ involves exactly $s$-sparse vectors. Note that if $\|X\|=1$, then $\downk(s)^{1/2}$ equals the smallest scaled singular value of a submatrix of $W^{1/2}X$ of dimension $s$. Similarly, $\upk(s)^{1/2}$ upper bounds the operator norm of $X$ when restricted to exactly $s$-sparse vectors.

Even though $W$ depends on the unknown $\theta_0$, it does not necessarily play a significant role in the above definitions. If $\|X\theta_0\|_\infty$ is bounded, then the true regression function $P_{\theta_0}(Y=1|X=x_i) =\Psi(X_i^T \theta_0)$ is bounded away from zero and one at the design points and $W$ is equivalent to the identity matrix $I_n$. One can then set $W = I_n$ in the above definitions by simply rescaling the constants. Note that estimation in classification problems is known to behave qualitatively differently near the boundary points 0 and 1, see, e.g. \cite{ray2016}.

When $W = I_n$, we recover the exact compatibility constants used in sparse linear regression \cite{castillo2015,rayszabo2019}. This is natural since when linearizing the logistic regression model, the likelihood asymptotically looks like that of a linear regression model with design matrix $W^{1/2}X$, see Section \ref{sec:proofs}. One therefore expects similar conditions with $X$ replaced by $W^{1/2}X$. For further discussion, see Chapter 6 of \cite{buhlmann2011} or Section 2.2 of \cite{castillo2015}; in particular, Lemma 1 of \cite{castillo2015} provides a concise relation between various notions of compatibility.

Another common condition in the sparsity literature is the \textit{mutual coherence} of $X$, which equals the largest correlation between its columns:
$$\mc (X) = \max_{1 \leq i\neq j\leq p} \frac{|\langle X_{\cdot i}, X_{\cdot j} \rangle|}{\|X_{\cdot i}\|_2 \|X_{\cdot j}\|_2}=\max_{1 \leq i\neq j\leq p} \frac{|(X^TX)_{ij}|}{\|X_{\cdot i}\|_2 \|X_{\cdot j}\|_2}.$$
Conditions of this nature have been used by many authors (see Section 2.2 of \cite{castillo2015} for references) and measure how far from orthogonal the matrix $X$ is. One can relate the present compatibility constants to the mutual coherence. 

\begin{lemma}
Suppose $\|X\theta_0\|_\infty\leq R$ is bounded and $\min_{1 \leq i \neq j \leq p} \frac{\|X_{\cdot i}\|_2}{\|X_{\cdot j}\|_2} \geq \eta$. Then for $C = C(R)$,
\begin{align*}
\upk (s) \leq 1+s \mc(X), \qquad \downk \geq C(R)(\eta^2 - 64s_0 \mc(X)), \qquad \downk(s) \geq C(R)(\eta^2-s \mc(X)).
\end{align*}
\end{lemma}

\begin{proof}
For $\theta$ an $s$-sparse vector, using Cauchy-Schwarz,
\begin{align*}
\|X\theta\|_2^2 = \sum_{j=1}^p \left( (X^TX)_{jj} \theta_j^2 + \sum_{k\neq j} \theta_j (X^TX)_{jk}\theta_k \right) & \leq \|X\|^2 \|\theta\|_2^2 + \mc(X)\|X\|^2 \|\theta\|_1^2 \\
& \leq (1+s \mc(X))\|X\|^2 \|\theta\|_2^2,
\end{align*}
so that $\upk (s) \leq 1+s\mc(X)$. For $R>0$, one has $\Psi(-R)\geq e^{-R}/2$ and $\Psi(R)\leq 1-e^{-R}/2$ , so that $\Psi(x_i^T\theta_0) \in [e^{-R}/2,1-e^{-R}/2]$. Using the definition \eqref{W matrix}, all diagonal entries of $W$ satisfy $W_{ii} \in [\delta_R,1/4]$ for $\delta_R = e^{-R}(2-e^{-R})/4>0$, so that $\|W^{1/2}X\theta\|_2^2 \geq \delta_R \|X\theta\|_2^2$. It thus suffices to prove the result with $W = I_n$ at the expense of the factor $\delta_R$. With $W = I_n$, arguing as in Lemma 1 of \cite{castillo2015} gives $\downk \geq \eta^2 - 64s_0 \mc(X)$ and $\downk (s) \geq \eta^2 - s \mc(X)$.
\end{proof}

If $\|X\theta_0\|_\infty$ is bounded and
\begin{equation}\label{mc}
s_0 = o(1/\mc(X)), \qquad \qquad \min_{1 \leq i \neq j \leq p} \frac{\|X_{\cdot i}\|_2}{\|X_{\cdot j}\|_2} \geq \eta > 0,
\end{equation}
namely the truth is sufficiently sparse and the column norms of $X$ are comparable, then $\upk(Ls_0) \leq 1+o(1)$, $\downk \gtrsim \eta^2-o(1)$ and $\downk (Ls_0) \gtrsim \eta^2 - o(1)$ for any $L>0$, as required for the results in this paper. Condition \eqref{mc} has been considered in \cite{rayszabo2019}, and thus the various examples in \cite{rayszabo2019} are also covered by our results, including:
\begin{itemize}
\item (Orthogonal design). If $X$ is an orthogonal matrix with $\langle X_{\cdot i}, X_{\cdot j} \rangle = 0$ for $i\neq j$ with suitably normalized column lengths $\|X_{\cdot i}\|_2 = \sqrt{n}$.
\item (IID responses). Suppose the original matrix entries are i.i.d. random variables $W_{ij}$ and set $X_{ij} = \sqrt{n}W_{ij} /\|W_{\cdot j}\|_2$, so that the columns are normalized to have length $\sqrt{n}$ . If $|W_{ij}| \leq C$ almost surely and $\log p = o(n)$, then \eqref{mc} holds for sparsity levels $s_0 = o(\sqrt{n/\log p})$. Similarly, if $Ee^{t|W_{ij}|^\gamma}<\infty$ for some $\gamma,t>0$ and $\log p = o(n^{\gamma/(4+\gamma)})$, then \eqref{mc} again holds for $s_0 = o(\sqrt{n/\log p})$. This covers the standard Gaussian random design $W_{ij}\sim^{iid}N(0,1)$ if $\log p = o(n^{1/3})$. See \cite{rayszabo2019} for details.
\item Rescale the columns as in the IID response model so that $\|X_{\cdot i}\|_2 = \sqrt{n}$ for all $i$. Then the $p\times p$ matrix $C=X^TX/n$ takes values one on its diagonal and $C_{ij}$, $i\neq j$, equals the correlation between columns $i$ and $j$. If either $C_{ij}=r$ for a constant $0<r<(1+cm)^{-1}$ and all $i\neq j$, or $|C_{ij}| \leq \tfrac{c}{2m-1}$ for every $i\neq j$, then $\mc(X) = \max_{i\neq j}C_{ij} = O(1/m)$ and so $\eqref{mc}$ holds for sparsity level $s_0 = o(m)$. Such matrices are studied in Zhao and Yu \cite{zhao2006}, who show that models up to dimension $m$ satisfy the `strong irrepresentability condition'.
\end{itemize}
For further details of why these satisfy \eqref{mc}, and hence our conditions, see Section 2.2 of \cite{rayszabo2019}.

\section{Non-asymptotic results and proofs}\label{sec:suppl:proofs}

This section contains the non-asymptotic formulations of all the technical results used in this paper, together with their proofs. These results imply the more digestible asymptotic formulations presented in Section \ref{sec:results}. To begin, recall the log-likelihood in model \eqref{model} based on data $Y = (Y_1,\dots,Y_n)$ is
\begin{equation}\label{likelihood}
\ell_{n,\theta}(Y) = \sum_{i=1}^n Y_i x_i^T \theta - g(x_i^T \theta) = \sum_{i=1}^n Y_i (X\theta)_i - g((X\theta)_i),
\end{equation}
where $g(t) = \log (1+e^t)$, $t\in \R$. We use the following notation for the first-order remainder of the Taylor expansion of the log-likelihood:
\begin{equation}\label{Ln}
\mL_{n,\theta}(y) := \ell_{n,\theta}(y) - \ell_{n,\theta_0}(y) - \nabla_\theta \ell_{n,\theta_0}(y)^T(\theta-\theta_0).
\end{equation}

\subsection{The Kullback-Leibler divergence between $Q^*$ and $\Pi(\cdot|Y)$}

To apply Lemma \ref{lem:KL_VB_thm}, we must bound $\KL(Q^*||\Pi(\cdot|Y))$ on the event $\mA_n$ given in \eqref{An event}. We do this by bounding the KL divergence between the posterior and a carefully selected element of the variational family. We choose a spike and slab distribution whose slab is centered at the least squares estimator of the linearized logistic likelihood approximation and whose covariance equals the (diagonalized) covariance of this estimator. This builds on ideas in \cite{rayszabo2019}, extending them to the nonlinear logistic regression model.

For a given model $S \subseteq \{1,\dots,p\}$, we write $X_S$ for the $n\times |S|$-submatrix of $X$ keeping only the columns $X_{\cdot i}$, $i\in S$, and $\theta_S\in \R^{|S|}$ for the vector $(\theta_i:i\in S)$.

\begin{lemma}\label{lem:KL}
Consider the event $\mA_n = \mA_n(t,L,M_1,M_2)$ in \eqref{An event} with $M_2 > L$. If $(4e)^{1/(M_2-L)} \leq p^{s_0}$, then the variational Bayes posterior $Q^*$ satisfies 
\begin{align*}
\KL (Q^*||\Pi(\cdot|Y))1_{\mA_n} \leq \zeta_n,
\end{align*}
where
\begin{align*}
\zeta_n & = s_0 \log p\left( L + \frac{9}{4}M_1^2 \upk (Ls_0) + \frac{\lambda L^{1/2} }{\|X\| \sqrt{\log p}} \left( 2M_1 + \frac{t L^{1/2} }{4\downk (Ls_0) \|X\| \sqrt{\log p}} + \frac{L^{1/2} }{\sqrt{\downk (1) \log p}} \right) \right)\\
& \qquad  + Ls_0\log \frac{1}{4\downk (Ls_0)} + \log (2e).
\end{align*}
\end{lemma}

\begin{proof}
Since the VB posterior $Q^*$ minimizes the KL objective \eqref{VB}, we have $\KL (Q^*||\Pi(\cdot|Y)) \leq \KL(Q||\Pi(\cdot|Y))$ for all $Q \in \Q$. It thus suffices to bound this last KL divergence for a suitably chosen element $Q \in \Q$, which may (and will) depend on the true unknown parameter $\theta_0$, on the event $\mA_n = \mA_n(t,L,M_1,M_2)$.

Recall that the posterior distribution is a mixture over all possible submodels and can thus be written
\begin{equation}\label{posterior}
\Pi(\cdot|Y) = \sum_{S \subseteq \{1,\dots,p\}} \hat{w}_S \Pi_S(\cdot|Y) \otimes \delta_{S^c},
\end{equation}
where the posterior model weights satisfy $0\leq \hat{w}_S \leq 1$ and $\sum_S \hat{w}_S = 1$ and $\Pi_S(\cdot|Y)$ denotes the posterior for $\theta_S \in \R^{|S|}$ in the restricted model with $Y_i \sim \text{Bin}(1,\Psi ((X_S\theta_S)_i))$, i.e. the logistic regression model \eqref{model} with $\theta_{S^c} = 0$.

\textbf{Choosing $Q'\in \Q$.} Recall that $E_{\theta_0}Y_i = \Psi(x_i^T\theta_0)$ and for a model $S \subseteq \{1,\dots,p\}$ set
\begin{equation}\label{VB_mean}
\mu_{S} = \theta_{0,S} + (X_{S}^T X_{S})^{-1} X_{S}^T(Y-E_{\theta_0}Y), \qquad \qquad \Sigma_{S} = (X_{S}^T X_{S})^{-1},
\end{equation}
and define the $|S| \times |S|$ diagonal matrix $D_{S}$ by
\begin{equation}\label{VB_cov}
(D_{S})_{jj} = \frac{1}{(\Sigma_{S}^{-1})_{jj}} = \frac{1}{(X_{S}^T X_{S})_{jj}},
\end{equation}
$j=1,\dots,|S|$. We choose as element of our variational family the distribution
\begin{align*}
Q'(\theta) = N_{S'}(\mu_{S'},D_{S'})(\theta_{S'}) \times \delta_{S'^c}(\theta_{S'^c}) = \prod_{j\in S'} N\left(\mu_{S',j},\tfrac{1}{(X_{S'}^T X_{S'})_{jj}}\right)(\theta_i) \prod_{j\in S'^c} \delta_0(\theta_i),
\end{align*}
for a model $S'\subseteq \{1,\dots,p\}$ satisfying the following three properties:
\begin{equation}\label{S' properties}
|S'| \leq L s_0, \qquad \|\theta_{0,S'^c}\|_2 \leq M_1\sqrt{s_0\log p}/\|X\|, \qquad \hat{w}_{S'} \geq (2e)^{-1} p^{-Ls_0},
\end{equation}
where $L,M_1$ are the constants in $\mA_n(t,L,M_1,M_2)$. Note that $Q'$ is indeed an element of the mean-field variational family $\Q$ in \eqref{VarFamily} with $\gamma_j = 1$ if $i\in S'$ and $\gamma_j = 0$ otherwise.

\textbf{Existence of $S'$ satisfying \eqref{S' properties}.} We first show that on the event $\mA_n$, there exists a subset $S'$ satisfying \eqref{S' properties}, so that our choice of $Q'$ is indeed valid. On $\mA_n$,
\begin{align*}
\Pi(\theta: \|\theta_{0,S_\theta^c}\|_2 > M_1\sqrt{s_0\log p}/\|X\||Y) & \leq \Pi(\theta: \|\theta-\theta_0\|_2 > M_1\sqrt{s_0\log p}/\|X\||Y) \\
& \leq e^{-M_2 s_0\log p} \to 0,
\end{align*}
so that the posterior model weights satisfy
$$\sum_{\substack{S: |S| \leq L s_0 \\ \|\theta_{0,S^c}\|_2 \leq M_1\sqrt{s_0\log p}/\|X\| }} \hat{w}_S \geq 3/4 - e^{-M_2 s_0\log p} \geq 1/2$$
on $\mA_n$, since $e^{-M_2 s_0\log p} \leq (4e)^{-M_2/(M_2-L)} \leq (4e)^{-1} \leq 1/4$ by assumption. Using ${p\choose s} \leq p^s/s!$, the number of elements in the last sum is bounded by
$$\sum_{S:|S| \leq L s_0} 1 \leq \sum_{s=0}^{L s_0} {p \choose s} \leq \sum_{s=0}^{Ls_0} \frac{p^s}{s!} \leq e p^{Ls_0},$$
which implies that on $\mA_n$ there exists a set $S'\subseteq \{1,\dots,p\}$ of size $|S'| \leq L s_0$ with $\|\theta_{0,S'^c}\|_2 \leq M_1\sqrt{s_0\log p}/\|X\|$ and with posterior probability $\hat{w}_{S'} \geq (2e)^{-1} p^{-Ls_0}$, i.e. satisfying \eqref{S' properties}.

\textbf{Reduction to the non-diagonal covariance case.}
Since $Q'$ is only absolutely continuous with respect to the $\hat{w}_{S'} \Pi_{S'}(\cdot|Y) \otimes \delta_{S'^c}$ term of the posterior \eqref{posterior},
\begin{align*}
\KL (Q' || \Pi(\cdot|Y)) &= E_{\theta \sim N_{S'}(\mu_{S'},D_{S'}) \otimes \delta_{S'^c}} \log \frac{dN_{S'}(\mu_{S'},D_{S'}) \otimes \delta_{S'^c}}{\hat{w}_{S'} d\Pi_{S'}(\cdot|Y)\otimes \delta_{S'^c} }(\theta) \\
& = \log(1/\hat{w}_{S'})+ \KL(  N_{S'}(\mu_{S'},D_{S'}) || \Pi_{S'}(\cdot|Y)),
\end{align*}
where the last KL divergence is over distributions in $\R^{|S'|}$. On $\mA_n$, $\log(1/\hat{w}_{S'}) \leq \log (2ep^{L s_0}) = \log(2e) + L s_0 \log p$ by \eqref{S' properties}. Writing $E_{\mu_{S'},D_{S'}}$ for the expectation under the law $\theta_{S'} \sim N_{S'}(\mu_{S'},D_{S'})$,
\begin{equation}\label{KL_S}
\KL(  N_{S'}(\mu_{S'},D_{S'}) || \Pi_{S'}(\cdot|Y)) = E_{\mu_{S'},D_{S'}} \left[ \log \frac{dN_{S'}(\mu_{S'},D_{S'})}{dN_{S'}(\mu_{S'},\Sigma_{S'})}(\theta_S) + \log \frac{dN_{S'}(\mu_{S'},\Sigma_{S'}) }{d\Pi_{S'}(\cdot|Y)}(\theta_S) \right],
\end{equation}
where again $\Sigma_{S'}= (X_{S'}^T X_{S'})^{-1}$. Using the formula for the KL divergence between two multivariate Gaussian distributions, the first term in the last display equals
$$\KL(N_{S'}(\mu_{S'},D_{S'})||N_{S'}(\mu_{S'},\Sigma_{S'})) = \frac{1}{2} \left( \log \frac{\det \Sigma_{S'}}{\det D_{S'}} - |S'| + \tr (\Sigma_{S'}^{-1} D_{S'}) \right).$$
Using the definitions \eqref{VB_mean}-\eqref{VB_cov} gives $\tr (\Sigma_{S'}^{-1} D_{S'}) = |S'|$. Turning to the determinants,
$$\det D_{S'}^{-1} = \prod_{j=1}^{|S'|} (X_{S'}^T X_{S'})_{jj} = \prod_{j\in S'} \|X_{\cdot j}\|_2^2 \leq \|X\|^{2|S'|}.$$
Let $\Lambda_\text{max}(A)$ and $\Lambda_\text{min}(A)$ denote the largest and smallest eigenvalues, respectively, of a square matrix $A$. Recall the diagonal matrix $W$ from \eqref{W matrix}, whose entries satisfy $W_{ii} \in (0,1/4]$. Using the variational characterization of the minimal eigenvalue of a symmetric matrix (\cite{horn2013}, p234),
\begin{equation}\label{min_eigenvalue}
\Lambda_\text{min} (X_{S'}^T X_{S'}) = \min_{u\in \R^{|S'|}:u\neq 0} \frac{u^T X_{S'}^T X_{S'} u}{\|u\|_2^2} \geq 4\min_{v\in \R^p: v\neq 0, v_{S'^c =0}} \frac{v^T X^T W Xv}{\|v\|_2^2} \geq 4\downk (|S'|)\|X\|^2.
\end{equation}
Since $\Sigma_{S'} = (X_{S'}^T X_{S'})^{-1}$ is positive definite, the last display implies
$$\det \Sigma_{S'} \leq \Lambda_\text{max}((X_{S'}^T X_{S'})^{-1})^{|S'|}  = (1/\Lambda_\text{min}(X_{S'}^T X_{S'}))^{|S'|} \leq \frac{1}{(4\downk (|S'|)\|X\|^2)^{|S'|}}.$$
Combining these bounds,
\begin{align*}
\KL(  N_{S'}(\mu_{S'},D_{S'}) ||N_{S'}(\mu_{S'},\Sigma_{S'})) & = \tfrac{1}{2} \log (\det \Sigma_{S'} \det D_{S'}^{-1}) \\
& \leq |S'| \log (1/(4\downk (|S'|)) \\
& \leq L s_0 \log (1/(4\downk (L s_0)).
\end{align*}
It thus remains to bound the second term in \eqref{KL_S}, which has non-diagonal covariance matrix $\Sigma_{S'}$.

\textbf{Bounding the non-diagonal covariance case.} One can check that $-\tfrac{1}{2} (\theta_{S'} - \mu_{S'})^T \Sigma_{S'}^{-1} (\theta_{S'}-\mu_{S'})$ equals
\begin{align*}
& -\tfrac{1}{2} (\theta_{S'}-\theta_{0,S'})^T X_{S'}^T X_{S'} (\theta_{S'}-\theta_{0,S'}) + (Y-E_{\theta_0}Y)^T X_{S'}(\theta_{S'} - \theta_{0,S'}) + C_{S'}(X,Y),
\end{align*}
where $C_{S'}(X,Y)$ does not depend on $\theta$. Let $\bar{\theta}_{S'}$ denote the extension of a vector $\theta_{S'} \in \R^{|S'|}$ to $\R^p$ with $\bar{\theta}_{S',j} = \theta_{S,j}$ for $j\in S'$ and $\bar{\theta}_{S',j} = 0$ for $j\not\in S'$. Since $(Y-E_{\theta_0}Y)^T X_{S'}(\theta_{S'} - \theta_{0,S'})=\nabla_\theta \ell_{n,\theta_0}(Y)^T(\bar{\theta}_{S'} - \bar{\theta}_{0,S'})$, the density function of the $N_{S'}(\mu_{S'},\Sigma_{S'})$ distribution is thus proportional to $e^{-\tfrac{1}{2}\|X_{S'}(\theta_{S'}-\theta_{0,S'})\|_2^2 +\nabla_\theta \ell_{n,\theta_0}(Y)^T(\bar{\theta}_{S'} - \bar{\theta}_{0,S'}) }$, $\theta_{S'} \in \R^{|S'|}$. Using Bayes formula and the Taylor expansion \eqref{Ln}, $\Pi_{S'}(\cdot|Y)$ has density proportional to
\begin{align*}
& \exp\left( \ell_{n,\bar{\theta}_{S'}}(Y) - \ell_{n,\theta_0}(Y) - \lambda\|\theta_{S'}\|_1\right) \\
& \propto \exp \left( \nabla_\theta \ell_{n,\theta_0}(Y)^T(\bar{\theta}_{S'} - \bar{\theta}_{0,S'}) + \mL_{n,\bar{\theta}_{S'}}(Y)-\lambda\|\theta_{S'}\|_1 \right).
\end{align*}
Using these representations of the two densities, the second term in \eqref{KL_S} can be rewritten as
\begin{equation*}
\begin{split}
&  E_{\mu_{S'},D_{S'}} \left[ \log \frac{ D_\Pi e^{-\frac{1}{2}\|X_{S'}(\theta_{S'}-\theta_{0,S'})\|_2^2 +\nabla_\theta \ell_{n,\theta_0}(Y)^T (\bar{\theta}_{S'} - \bar{\theta}_{0,S'}) -\lambda \|\theta_{0,S'}\|_1} }{ D_N e^{\nabla_\theta \ell_{n,\theta_0}(Y)^T(\bar{\theta}_{S'} - \bar{\theta}_{0,S'}) + \mL_{n,\bar{\theta}_{S'}}(Y) -\lambda\|\theta_{S'}\|_1} } \right] \\
& \quad = E_{\mu_{S'},D_{S'}} \left[ -\tfrac{1}{2}\|X_{S'}(\theta_{S'}-\theta_{0,S'})\|_2^2 - \mL_{n,\bar{\theta}_S}(Y) \right]\\
& \qquad + \lambda E_{\mu_{S'},D_{S'}} (\|\theta_{S'}\|_1 - \|\theta_{0,S'}\|_1) +\log (D_\Pi/D_N) \\
& \quad =: (I) + (II) + (III),
\end{split}
\end{equation*}
where the normalizing constants are $D_\Pi = \int_{\R^{|S'|}} e^{\nabla_\theta \ell_{n,\theta_0}(Y)^T(\bar{\theta}_{S'} - \bar{\theta}_{0,S'}) + \mL_{n,\bar{\theta}_{S'}}(Y) -\lambda\|\theta_{S'}\|_1} d\theta_{S'}$ and $D_N = \int_{\R^{|S'|}} e^{-\frac{1}{2}\|X_{S'}(\theta_{S'}-\theta_{0,S'})\|_2^2 +\nabla_\theta \ell_{n,\theta_0}(Y)^T (\bar{\theta}_{S'} - \bar{\theta}_{0,S'}) -\lambda \|\theta_{0,S'}\|_1} d\theta_{S'}$. We now bound $(I)-(III)$ in turn.

$(I)$: Using the likelihood \eqref{likelihood} and the mean-value form of the remainder in the Taylor expansion \eqref{Ln}, for $\xi_i$ between $x_i^T\bar{\theta}_{S'}$ and $x_i^T \theta_0$,
\begin{equation}\label{Ln_LB}
\begin{split}
\mL_{n,\bar{\theta}_{S'}} & = -\frac{1}{2}\sum_{i=1}^n g''(\xi_i)|x_i^T(\bar{\theta}_{S'}-\theta_0)|^2 \\
& \geq - \frac{1}{8} \sum_{i=1}^n 2|x_i^T(\bar{\theta}_{S'}-\bar{\theta}_{0,S'})|^2 + 2|x_i^T(\bar{\theta}_{0,S'}-\theta_0)|^2 \\
& = -\frac{1}{4}\|X_{S'}(\theta_{S'}-\theta_{0,S'})\|_2^2 - \frac{1}{4} \|X_{S'^c} \theta_{0,S'^c}\|_2^2,
\end{split}
\end{equation}
so that $(I)$ is bounded by $\|X_{S'^c} \theta_{0,S'^c}\|_2^2/4$. On $\mA_n$, using \eqref{S' properties},
$$\|X_{S'^c}\theta_{0,S'^c}\|_2^2 = \|X \bar{\theta}_{0,S'^c}\|_2^2 \leq \upk (| S_0\cap S'^{c}|) \|X\|^2 \|\theta_{0,S'^c}\|_2^2 \leq \upk (s_0) M_1^2 s_0 \log p,$$
so that $(I) \leq \upk (s_0) M_1^2 s_0 \log p/4$. 

$(II)$: Under the expectation $E_{\mu_{S'},D_{S'}}$, we have the equality in distribution $\theta_{S'} - \theta_{0,S'} =^d (X_{S'}^T X_{S'})^{-1} X_{S'}^T(Y-E_{\theta_0}Y) + Z$, where $Z\sim N_{S'}(0,D_{S'})$. Applying the triangle inequality and Cauchy-Schwarz,
\begin{align*}
(II) & \leq \lambda \|(X_{S'}^T X_{S'})^{-1} X_{S'}^T(Y-E_{\theta_0}Y) \|_1 + \lambda E_{\mu_{S'},D_{S'}}\|Z\|_1 \\
& \leq \lambda |S'|^{1/2} (\|(X_{S'}^T X_{S'})^{-1} X_{S'}^T(Y-E_{\theta_0}Y) \|_2 + \tr (D_{S'})^{1/2}),
\end{align*}
since by Jensen's inequality $E\|Z\|_2\leq  (E\|Z\|_2^2)^{1/2} = \tr(D_{S'})^{1/2}$. Using the definition \eqref{VB_cov}, for $e_j$ the $j^{th}$ unit vector in $\R^p$,
$$\tr(D_{S'}) = \sum_{j=1}^{|S'|} \frac{1}{(X_{S'}^T X_{S'})_{jj}} = \sum_{j\in S'} \frac{1}{\|Xe_j\|_2^2} \leq \sum_{j\in S'} \frac{1}{\downk (1)\|X\|^2} = \frac{|S'|}{\downk (1) \|X\|^2}.$$
The matrix operator norm (from $\R^{|S'|}$ to $\R^{|S'|}$) of $(X_{S'}^T X_{S'})^{-1}$ equals its largest eigenvalue, which is bounded by $1/(4\downk (|S'|)\|X\|^2)$ using \eqref{min_eigenvalue}. Recalling that $\nabla_\theta \ell_{n,\theta_0}(Y) = X^T (Y-E_{\theta_0}Y)$, on the event $\mA_n$ the first term in the second to last display is therefore bounded by
\begin{align*}
\frac{\lambda |S'|^{1/2}}{4\downk (|S'|) \|X\|^2} \|X_{S'}^T (Y - E_{\theta_0}Y) \|_2 \leq \frac{\lambda |S'|}{4\downk (|S'|) \|X\|^2} \|X^T (Y - E_{\theta_0}Y) \|_\infty \leq  \frac{\lambda Ls_0 t}{4\downk (Ls_0) \|X\|^2}.
\end{align*}
We have thus shown that 
$$(II) \leq \frac{\lambda Ls_0}{\|X\|} \left( \frac{t}{4\downk (Ls_0)\|X\|} + \frac{1}{\downk(1)^{1/2}}\right).$$

$(III)$: It remains to control the ratio of normalizing constants $\log (D_\Pi/D_N)$. Define
$$B_{S'}  = \{ \theta_{S'} \in \R^{|S'|}: \|\theta_{S'} - \theta_{0,S'}\|_2 \leq 2M_1 \sqrt{s_0\log p}/\|X\| \}.$$
On $\mA_{n}$, using \eqref{posterior} and \eqref{S' properties},
\begin{align*}
\Pi_{S'} (B_{S'}^c|Y) &\leq \frac{\hat{w}_{S'}}{\hat{w}_{S'}} \Pi_{S'} (\theta_{S'}\in \R^{|S'|}: \| \bar{\theta}_{S'}-\theta_0\|_2 >2M_1\sqrt{s_0\log p}/\|X\| -\|\theta_{0,S'^c}\|_2 |Y)\\
&\leq  \hat{w}_{S'}^{-1}\Pi (\theta \in \R^p: \|\theta-\theta_0\|_2 > M_1\sqrt{s_0\log p}/\|X\| |Y)\\
&\leq 2ep^{Ls_0} e^{-M_2 s_0\log p}= 2e^{1-(M_2-L)s_0\log p} \leq 1/2,
\end{align*}
where the last inequality follows from rearranging the assumption $(4e)^{1/(M_2-L)} \leq p^{s_0}$. Using Bayes formula, this gives
$$\Pi_{S'} (B_{S'}|Y) 1_{\mA_n} = \frac{\int_{B_{S'}} e^{\ell_{n,\bar{\theta}_{S'}}(Y) - \ell_{n,\theta_0}(Y)-\lambda \|\theta_{S'}\|_1} d\theta_{S'}}{\int_{\R^{|S'|}} e^{\ell_{n,\bar{\theta}_{S'}}(Y) - \ell_{n,\theta_0}(Y)-\lambda \|\theta_{S'}\|_1} d\theta_{S'} } 1_{\mA_n} \geq \frac{1}{2} 1_{\mA_n}.$$
By \eqref{Ln} the denominator in the last display equals $e^{\nabla_\theta \ell_{n,\theta_0}(Y)^T(\bar{\theta}_{0,S'} - \theta_0)}D_\Pi$, which implies that on $\mA_n$, $D_\Pi \leq 2 \int_{B_{S'}} e^{\nabla_\theta \ell_{n,\theta_0}(Y)^T(\bar{\theta}_{S'} - \bar{\theta}_{0,S'}) + \mL_{n,\bar{\theta}_{S'}}(Y) -\lambda\|\theta_{S'}\|_1} d\theta_{S'}.$ Thus on $\mA_n$,
\begin{align*}
\log \frac{D_\Pi}{D_N} & \leq \log \frac{2 \int_{B_{S'}} e^{\nabla_\theta \ell_{n,\theta_0}(Y)^T(\bar{\theta}_{S'} - \bar{\theta}_{0,S'}) + \mL_{n,\bar{\theta}_{S'}}(Y) -\lambda\|\theta_{S'}\|_1} d\theta_{S'}}{\int_{B_{S'}} e^{-\frac{1}{2}\|X_{S'}(\theta_{S'}-\theta_{0,S'})\|_2^2 +\nabla_\theta \ell_{n,\theta_0} (\bar{\theta}_{S'} - \bar{\theta}_{0,S'}) -\lambda \|\theta_{0,S'}\|_1} d\theta_{S'}} \\
& \leq \log \left( \sup_{\theta_{S'} \in B_{S'}} e^{\mL_{n,\bar{\theta}_{S'}}(Y)+\frac{1}{2}\|X_{S'}(\theta_{S'}-\theta_{0,S'})\|_2^2 + \lambda \|\theta_{0,S'}\|_1 - \lambda\|\theta_{S'}\|_1} \right) + \log 2 \\
& \leq \sup_{\theta_{S'} \in B_{S'}} \mL_{n,\bar{\theta}_{S'}}(Y)+\tfrac{1}{2}\|X_{S'}(\theta_{S'}-\theta_{0,S'})\|_2^2 + \lambda \|\theta_{S'}-\theta_{0,S'}\|_1 + \log 2.
\end{align*}
Now $\mL_{n,\bar{\theta}_{S'}}(Y) <0$ by \eqref{Ln_LB} since $g'' > 0$. Using Cauchy-Schwarz and the definition of $B_{S'}$, on $\mA_n$,
\begin{align*}
(III) = \log \frac{D_\Pi}{D_N} & \leq \sup_{\theta_{S'} \in B_{S'}} \tfrac{1}{2} \upk (|S'|)\|X\|^2 \|\theta_{S'}-\theta_{0,S'}\|_2^2 + \lambda |S'|^{1/2} \|\theta_{S'}-\theta_{0,S'}\|_2 + \log 2 \\
& \leq 2\upk (Ls_0) M_1^2 s_0 \log p + \frac{2M_1 \lambda L^{1/2} s_0 \sqrt{\log p}}{\|X\|} + \log 2. 
\end{align*}
Combining all of the above bounds gives the result.
\end{proof}

\subsection{Contraction results}

The second part to applying Lemma \ref{lem:KL_VB_thm} is showing that on an event, the desired sets have all but exponentially small posterior probability. This involves using results on dimension selection and posterior contraction from high-dimensional Bayesian statistics, especially Atchad\'e \cite{atchade2017} and Castillo et al. \cite{castillo2015}. The following results follow closely the proofs in \cite{atchade2017}, but we reproduce them here for convenience, since in that paper they are not stated or proved in the exponential form needed to apply Lemma \ref{lem:KL_VB_thm}. We are also able to simplify certain technical conditions and streamline some proofs. Note that his results, including the definitions of the compatibility constants, match when $\|X\| \sim \sqrt{n}$.

We next introduce some notation from \cite{atchade2017}, used throughout this section. A continuous function $r:[0,\infty) \to [0,\infty)$ is called a \textit{rate function} if it is strictly increasing, $r(0) = 0$ and $\lim_{x\downarrow 0} r(x)/x = 0$. For a rate function $r$ and $a \geq 0$, define
\begin{equation}\label{eq:phi}
\phi_r(a)=\inf\{ x>0: r(z)\geq az,\, \text{for all $z\geq x$} \},
\end{equation}
with the convention $\inf \emptyset = \infty$.
Let $B(\Theta, M)=\{\theta\in \theta_0+\Theta:\,\|\theta-\theta_0\|_2\leq M \}$ denote the $\ell_2$-ball of radius $M>0$ centered at $\theta_0$ with elements in $\theta_0+\Theta$. For $\eps>0$, we denote by $D(\eps,B(\Theta, M))$ the $\eps$-packing number of $B(\Theta, M)$, namely the maximal number of points in $B(\Theta,M)$ such that the $\ell_2$ distance between any two points is at least $\eps$.

The following result bounds the posterior probability of selecting a model of size larger than a multiple of the true model size.

\begin{lemma}[Theorem 4(1) of \cite{atchade2017}]\label{lem: thm4(1)}
Suppose the prior satisfies \eqref{prior_cond} and \eqref{lambda}, $p^{A_4}\geq 8A_2$, and that $\|X\| \geq (64/3) \alpha s_0\sqrt{\log p} /\underline{\kappa}$. Then for any $L>1$,
\begin{align*}
E_{\theta_0}[\Pi \left( \theta\in \R^p: |S_\theta| \geq L s_0 \mid Y \right) 1_{\mA_{n,1}(\lambda/2) }]   \leq 2 \exp \left(-s_0 \log p \left[ L\Big(A_4-\frac{\log(4A_2)}{\log p}\Big)-C \right]  \right),
\end{align*}
where $C=1 + \frac{4\alpha^2}{\underline{\kappa}} +  \log(4+\overline{\kappa}(s_0)/\log p)+(1+\frac{1}{s_0})(A_4-\frac{\log(4A_2)}{\log p})$.
\end{lemma}

\begin{proof}
By Lemma \ref{lem:thm:3}(1), for any $k\geq 0$,
\begin{equation*}
E_{\theta_0} [\Pi(\theta: |S_\theta| \geq s_0 + k|Y)1_{\mA_{n,1} (\lambda/2)  }] \leq 2e^{a} \left( 4+\frac{\overline{\kappa}(s_0)\|X\|^2}{\lambda^2}  \right)^{s_0} {p \choose s_0} \left( \frac{4A_2}{p^{A_4}}\right)^k,
\end{equation*}
where $a = -\tfrac{1}{2} \inf_{x>0}[\tfrac{\downk \|X\|^2 x^2}{1+4s_0^{1/2}\|X\|_\infty x} -4\lambda s_0^{1/2} x]$. It remains to simplify the right-hand side.

One can check that for $\tau,b,c>0$, $\inf_{x>0}[\tfrac{\tau x^2}{1+bx} - cx] \geq -\tfrac{c^2}{4\tau^{1/2}(\tau-cb)^{1/2}} \geq - \tfrac{c^2}{2\tau}$ if $\tau \geq 4bc/3$. In our setting, this condition equals $\|X\|^2 \underline{\kappa} \geq (64/3) \lambda s_0\|X\|_\infty $, which holds by assumption. This yields $a \leq \tfrac{4\lambda^2 s_0}{\|X\|^2 \underline{\kappa}}$. Using  the upper bound ${p \choose s_0} \leq p^{s_0}$, setting $k = \lfloor (L-1)s_0 \rfloor$ and using that $(L-1)s_0 - 1 \leq \lfloor (L-1)s_0 \rfloor \leq (L-1)s_0$ gives the  result.
\end{proof}

The next result is the analogous version of Theorem 4(2) in \cite{atchade2017} with the exponential bounds we require here. It provides a contraction rate for posterior models of a given size.

\begin{lemma}\label{lem: thm4(2)}
Suppose the prior satisfies \eqref{prior_cond} and \eqref{lambda}, $p^{A_4}\geq 8A_2$, and that $\|X\| \geq 50\alpha(L+2)s_0\sqrt{\log p}\|X\|_\infty/\underline{\kappa}((L+1)s_0)$ for some $L>0$. Then for any $\theta_0\in \R^p$, and $M\geq \max( 25\alpha,(1+A_3)/16)$,
\begin{align*}
& E_{\theta_0} \Big[\Pi(\theta \in \R^p: |S_\theta| \leq Ls_0, \|\theta-\theta_0\|_2 \geq  \frac{8M \sqrt{(L+2)s_0\log p}}{\underline{\kappa}((L+1)s_0)\|X\| } |Y)1_{\mA_{n,1}(\|X\|\sqrt{\log p})} \Big] \\
& \qquad \qquad \leq 6  e^{ -s_0\log p (8LM-C_L)},
\end{align*}
where $C_L=\max(L(1+\tfrac{\log(24)}{\log p}),\tilde{C}_p)$ and $\tilde{C}_p=\frac{\log A_1+ \log \left( 1+4\alpha^2 \log p\right)}{\log p}$. 
\end{lemma}

While $\tilde{C}_p$ is not a true constant since it depends on $p$, we write it as such since it is asymptotically negligible. As $p\to\infty$, we have $\tilde{C}_p \to 0$ and $C_L \to L$.

\begin{proof}
We write $\downk_L = \downk((L+1)s_0)$ during this proof to ease notation. By Lemma \ref{lem:thm:3}(2), for any $M >2$,
\begin{equation}\label{eq_contrac_bound}
\begin{split}
& E_{\theta_0} \Pi(\theta\in \theta_0+\bar\Theta_L:\, \|\theta-\theta_0\|_2>M \eps|Y)1_{\mA_{n,1} (\|X\|\sqrt{\log p}) } \\
& \quad \leq \sum_{j\geq 1}D_je^{-r(\frac{j M\eps}{2})/8}
+2 {p \choose s_0} \Big(\frac{p^{A_3}}{A_1}\Big)^{s_0}\Big(1+\frac{4\lambda^2}{\bar{\kappa}(s_0) \|X\|^2}\Big)^{s_0}\sum_{j\geq 1}e^{-r(\frac{j M\eps}{2})/8}e^{3\lambda c_0j M\eps},
\end{split}
\end{equation}
where the quantities in \eqref{eq_contrac_bound} are defined in that lemma. Note that for $\theta$ satisfying $|S_\theta| \leq Ls_0$, then $|S_{\theta-2\theta_0}| \leq (L+1)s_0$, so that $\theta - \theta_0 \in \bar{\Theta}_L$. We may thus further restrict the set in the last display to $\{\theta: |S_\theta| \leq Ls_0\}$, as in the posterior probability in the lemma. We first compute $\eps$ and then simplify the right-hand side of \eqref{eq_contrac_bound}.

Recall that we may take as rate any $\eps \geq \phi_r(2\eta_L)$ for $r$ the rate function in Lemma \ref{lem:thm:3}(2). For a rate function $r(x) = \tfrac{\tau x^2}{1+bx}$, $\tau,b>0$, the inequality $r(x) \geq ax$ is equivalent to $x((\tau-ab)x-a) \geq 0$. Using the definition \eqref{eq:phi} thus gives $\phi_r(a) = \tfrac{a}{\tau-ab}$. Setting $\tau = \downk_L \|X\|^2 $ and $b = \|X\|_\infty \sqrt{(L+1)s_0}/2$ as in Lemma \ref{lem:thm:3}(2), and using our assumption $\tfrac{1}{2}\|X\|^2 \downk_L \geq 25\alpha (L+2)s_0 \|X\| \sqrt{\log p} \|X\|_\infty\geq \eta_L \sqrt{(L+1)s_0}\|X\|_\infty$, we get
\begin{align*}
\phi_r(2\eta_L) & = \frac{2\eta_L}{\downk_L \|X\|^2 - \eta_L \|X\|_\infty \sqrt{(L+1)s_0}} \leq \frac{4\eta_L}{\downk_L \|X\|^2} = \frac{8\sqrt{(L+2)s_0\log p}}{\downk_L \|X\|} =: \eps
\end{align*}

Turning to the right-hand side of \eqref{eq_contrac_bound}, note that $c_0 = \sup_{v\in \bar{\Theta}_L} \|v\|_1 /\|v\|_2 \leq \sqrt{(L+2)s_0}$. Arguing as on p. 29-30 of \cite{atchade2017} gives $\sum_{j\geq 1} D_j e^{-r(jM\eps/2)/8} \leq 2\exp\big( (L+2)s_0\log p[1+\tfrac{\log(24)}{\log p} -8M]\big)$. Similarly, setting $x=jM\eps/2$,
\begin{align*}
3\lambda \sqrt{(L+2)s_0} jM \eps - \frac{1}{8}r(jM\eps/2) & = -\frac{x}{8} \left( \frac{\|X\|^2 \downk_L x}{1+\tfrac{1}{2}\sqrt{(L+1)s_0}\|X\|_\infty x}  - 48\lambda \sqrt{(L+2)s_0}\right) \\
& \leq -\frac{x}{8} \left( \frac{\|X\|^2 \downk_L \tfrac{M\eps}{2}}{1+\tfrac{1}{2}\sqrt{(L+1)s_0}\|X\|_\infty \tfrac{M\eps}{2}}  - 48\lambda \sqrt{(L+2)s_0}\right) \\
& \leq -\frac{\lambda\sqrt{(L+2)s_0}x}{4}
\end{align*}
as long as
$$\frac{\|X\|^2 \downk_L \tfrac{M\eps}{2}}{1+\tfrac{1}{2}\sqrt{(L+1)s_0}\|X\|_\infty \tfrac{M\eps}{2}}  \geq 50 \lambda \sqrt{(L+2)s_0}.$$
We show the last display holds under the present assumptions. Since $\sqrt{(L+1)s_0}\|X\|_\infty\eps \leq \tfrac{8 (L+2)s_0\sqrt{\log p}\|X\|_\infty}{\|X\| \downk_L}\leq 8/(50\alpha)$ by assumption, the left-hand side is lower bounded by $\tfrac{\|X\|^2 \downk_L M \eps}{2+4M/(50\alpha)} \geq (50\alpha/8)\|X\|^2 \downk_L \eps$ for $M\geq 25\alpha$. Since $\lambda \leq \alpha \|X\| \sqrt{\log p}$ by assumption, the last display holds following from
$$\frac{(50\alpha/8)\|X\|^2 \downk_L \eps}{50 \sqrt{(L+2)s_0}}  = \alpha \|X\| \sqrt{\log p}.$$
This implies
$$\sum_{j\geq 1} e^{-\frac{1}{8}r(jM\eps/2)} e^{3\lambda c_0 jM \eps} \leq \sum_{j\geq 1} e^{-\frac{jM\lambda \sqrt{(L+2)s_0}\eps}{8}} \leq 2e^{-8M (L+2)s_0\log p},$$
where the last inequality again follows by the same argument on p. 29-30 of \cite{atchade2017}.

Summing up these bounds and using ${p \choose s_0} \leq p^{s_0}$ and $\bar{\kappa}(s_0)\geq \bar{\kappa}(1)=1$, the right-hand side of \eqref{eq_contrac_bound} is bounded by
\begin{align*}
& 2\exp \left( (L+2)s_0\log p\left[1+\tfrac{\log(24)}{\log p} -8M \right]\right) \\
& \quad + 4\exp \left( (1+A_3)s_0 \log p + s_0 \log \left( 1+4\alpha^2 \log p \right)-s_0\log A_1-8M(L+2)s_0\log p  \right)\\
&\leq 6\exp\Big(-s_0\log p\Big[8LM- \max(L(1+\tfrac{\log(24)}{\log p}),\tilde{C}_p)) \Big]  \Big).
\end{align*}
\end{proof}

Combining the last two lemmas yields the contraction rate result with exponential bounds.

\begin{lemma}\label{lem: combined:contraction}
Suppose the prior satisfies \eqref{prior_cond} and \eqref{lambda}, and $p^{A_4}\geq 8A_2$. If for $K>0$, the design matrix satisfies condition \eqref{cond:design} with $L=L_K=\frac{K+C}{A_4-\log(4A_2)/\log p}$, with $C$ the constant in Lemma \ref{lem: thm4(1)}, then for any $\theta_0\in \R^p$,
\begin{align*}
E_{\theta_0} \left[\Pi \left( \theta\in \R^p:\|\theta-\theta_0\|_2 \geq  C_K \frac{\sqrt{s_0\log p}}{\|X\| } \bigg|Y \right)1_{\mA_{n,1}(\|X\|\sqrt{\log p})} \right] \leq 8  e^{ -Ks_0\log p },
\end{align*}
where $C_K= \frac{8M_K\sqrt{L_K+2}}{\underline\kappa((L_K+1)s_0)}$, $M_K= \max(25\alpha, \frac{1+A_3}{16}, \frac{K+\tilde{C}_p}{8L_K}+\frac{1}{8}+\frac{\log (24)}{8\log p})$ and $\tilde{C}_p$ is given in Lemma \ref{lem: thm4(2)}.
\end{lemma}

If all the compatibility constants are bounded away from zero and infinity, the constants in Lemma \ref{lem: combined:contraction} scale like $L_K \sim K$, $M_K \sim \sqrt{K}$ and $C_K \sim K$ as $K\to \infty$. We now have the required event to apply Lemma \ref{lem:KL_VB_thm}.

\begin{lemma}\label{lem:prob}
Suppose the prior satisfies \eqref{prior_cond} and \eqref{lambda}, and $p^{A_4}\geq 8A_2$. Set $L=\frac{1+C}{A_4-\log(4A_2)/\log p}$, where $C$ is the constant in Lemma \ref{lem: thm4(1)}, and assume the design matrix satisfies \eqref{cond:design} for this $L$. Then for $t=\|X\|\sqrt{\log p}$, $M_2 = 2L$ and $M_1 = C_{3L}$, where $C_{3L}$ is the constant in Lemma \ref{lem: combined:contraction} with $K=3L$, 
and any $\theta_0\in \R^p$,
\begin{align*}
P_{\theta_0}\Big( \mA_n (t,L,M_1,M_2)^c  \Big)\leq 2/p + (8/3)p^{-s_0} + 8p^{-Ls_0},
\end{align*}
where $\mA_n(t,L,M_1,M_2)$ is defined in \eqref{An event}.
\end{lemma}

\begin{proof}
Using a union bound and the definition \eqref{An event},
\begin{align*}
P_{\theta_0}\Big( \mA_n (t,L,M_1,M_2)^c  \Big)&\leq P_{\theta_0}( \mA_{n,1} (\|X\|\sqrt{\log p}) ^c )+ P_{\theta_0}( \mA_{n,2} (L)^c\cap \mA_{n,1} (\|X\|\sqrt{\log p})  )\\
&\qquad\quad+ P_{\theta_0}(\mA_{n,3} (M_1,M_2)^c\cap \mA_{n,1} (\|X\|\sqrt{\log p})).
\end{align*}
Since $\tfrac{\partial}{\partial \theta_j} \ell_{n,\theta_0}(Y) = \sum_{i=1}^n (Y_i-g'(x_i^T\theta_0)) X_{ij}$, by Hoeffding's inequality,
\begin{align*}
P_{\theta_0} (\mA_{n,1}( t)^c) & = P_{\theta_0} \left( \max_{1\leq j \leq p} \left| \sum_{i=1}^n (Y_i - g'(x_i^T\theta_0))X_{ij} \right| >t \right) \\
& \leq 2 \sum_{j=1}^p e^{-\frac{2t^2}{\|X_{\cdot j}\|_2^2 }} \leq 2pe^{-\frac{2t^2}{\|X\|^2}} = \frac{2}{p}.
\end{align*}
Applying Markov's inequality and Lemma \ref{lem: thm4(1)} with the present choice of $L$, the second term is bounded by
$$(4/3) E_{\theta_0} \left[ \Pi ( \theta \in \R^p: |S_\theta| > Ls_0 |Y) 1_{\mA_{n,1}(\|X\|\sqrt{\log p})} \right] \leq  (8/3)e^{-s_0\log p}.$$
Similarly, using Markov's inequality and Lemma \ref{lem: combined:contraction} with $K=3L$, the third term is bounded by
$$e^{M_2s_0\log p } E_{\theta_0} \left[ \Pi( \theta \in \R^p: \|\theta-\theta_0\|_2 \geq   M_1\sqrt{s_0\log p}/\|X\|  |Y) 1_{\mA_{n,1}(\|X\|\sqrt{\log p})} \right] \leq  8e^{- Ls_0\log p }.$$
\end{proof}

The following is a simplified version of Theorem 3 of Atchad\'e \cite{atchade2017}, which applies to general settings, tailored to the sparse high-dimensional logistic regression model. It gives high level technical conditions under which one can control (1) the posterior model dimension and (2) the posterior $\ell_2$ norm for models of restricted dimension.

\begin{lemma}\label{lem:thm:3}
Suppose the prior satisfies \eqref{prior_cond} and $p^{A_4}\geq 8A_2$.
\begin{itemize}
\item[(1)] For any integer $k\geq 0$,
\begin{align*}
E_{\theta_0} \Pi\Big( \theta\in\mathbb{R}^p:\, |S_\theta|\geq s_0+k  |Y \Big)1_{ \mA_{n,1} (\lambda/2)}\leq 2e^a \Big(4+\frac{\overline{\kappa}(s_0)\|X\|^2}{\lambda^2} \Big)^{s_0} {p \choose s_0}\Big( \frac{4A_2}{p^{A_4}}\Big)^k,
\end{align*}
where
$$a=-\frac{1}{2}\inf_{x>0} \left[ \frac{\downk \|X\|^2 x^2}{1+4s_0^{1/2}\|X\|_\infty x}-4\lambda \sqrt{s_0}x \right].$$
\item[(2)] For $L>0$, set $\bar{\Theta}_L = \{\theta \in \R^p: |S_{\theta - \theta_0}| \leq (L+1)s_0\}$ and define the rate function $r(x) = \tfrac{\downk ((L+1)s_0)\|X\|^2 x^2}{1+\|X\|_\infty \sqrt{(L+1)s_0}x/2}$. Further set $\eta_L = 2\sqrt{(L+2)s_0}\|X\|\sqrt{\log p}$ and $\eps=\phi_r(2\eta_L)$, where $\phi_r$ uses the same rate function $r$ and is defined in \eqref{eq:phi}. Then for any $M_0>2$,
\begin{align*}
& E_{\theta_0} \Pi(\theta\in \theta_0+\bar\Theta_L:\, \|\theta-\theta_0\|_2>M_0 \eps|Y)1_{\mA_{n,1} (\|X\|\sqrt{\log p}) } \\
& \quad \leq \sum_{j\geq 1}D_je^{-r(\frac{j M_0\eps}{2})/8}
+2 {p \choose s_0} \Big(\frac{p^{A_3}}{A_1}\Big)^{s_0}\Big(1+\frac{4\lambda^2}{\overline{\kappa}(s_0) \|X\|^2}\Big)^{s_0}\sum_{j\geq 1}e^{-r(\frac{j M_0\eps}{2})/8}e^{3\lambda c_0j M_0\eps},
\end{align*}
where $c_0=\sup_{u\in\bar\Theta_L}\sup_{v\in\bar\Theta_L,\, \|v\|_2=1}|\langle \text{sign}(u),v\rangle|$ and $D_j= D\big(\frac{jM_0\eps}{2},B(\bar{\Theta}_L,(j+1)M_0\eps)\big)$.
\end{itemize}
\end{lemma}

\begin{proof}
This is a combination of Theorems 3 and 4 in \cite{atchade2017}. In particular, we verify that certain technical assumptions of that result hold automatically in the logistic regression model, giving the simpler result above. Firstly note that Assumptions H1-H3 of \cite{atchade2017} are satisfied (H1) by definition, (H2) since $\theta \mapsto \ell_{n,\theta}$ in \eqref{likelihood} is concave and differentiable and (H3) by \eqref{prior_cond}.

For $y\in \{0,1\}^n$ data in model \eqref{model}, $\mL_{n,\theta}$ defined in \eqref{Ln}, $\Theta_0 = \{\theta: S_\theta \subseteq S_{\theta_0}\}$ and $r$ some rate function, define
\begin{align*}
\mathcal{N}&=\left\{ \theta\in\mathbb{R}^p:\, \theta\neq0,\,\text{and}\, \sum_{i\in S_0^c}|\theta_i|\leq 7\|\theta_{S_0}\|_1  \right\}, \nonumber\\
\check{\mE}_{n,1}(\mathcal{N},r)&=\left\{y\in \{0,1\}^n:\ \forall \theta\in\theta_0+\mathcal{N}: \mL_{n,\theta}(y) \leq-\tfrac{1}{2}r(\|\theta-\theta_0\|_2)  \right\},\nonumber\\
\hat{\mE}_{n,1}(\Theta_0,\bar{L})&= \left\{y\in \{0,1\}^n: \forall\theta\in \theta_0+\Theta_0,\,
  \mL_{n,\theta}(y)\geq-\tfrac{\bar{L}}{2}\|\theta-\theta_0\|_2^2\right\},\nonumber\\
\mE_{n,0}(\Theta,\lambda)&=\left\{ y\in\{0,1\}^n:\, \sup_{u\in\Theta, \|u\|_2=1} |\langle\nabla_\theta \log \ell_{n,\theta_0}(y),u\rangle|\leq \tfrac{\lambda}{2} \right\},\label{eq: notations:atchade2}
\end{align*}
where $\bar{L}>0$ and $\lambda>0$ is the regularization parameter in the prior \eqref{prior}. This matches the notation in \cite{atchade2017} (except note his $\rho$ is our $\lambda$), where it is shown the theorem's two conclusions hold under various choices of parameters in the last displays.

Part (1): \cite{atchade2017} considers the event $\mE_{n,0}(\R^p,\lambda) \cap \hat{\mE}_{n,1}(\Theta_0,\bar{L}) \cap \check{\mE}_n(\mathcal{N},r)$, which we now simplify. Arguing as on p27 of \cite{atchade2017} yields
$$\mL_{n,\theta}(y) \leq -\sum_{i=1}^n g''(x_i^T\theta_0) \frac{|x_i^T(\theta-\theta_0)|^2}{2+|x_i^T(\theta-\theta_0)|},$$
For $\theta - \theta_0\in \cal N$, 
$$|x_i^T(\theta-\theta_0)| \leq \|X\|_\infty\|\theta-\theta_0\|_1 \leq 8\|X\|_\infty s_0^{1/2} \|\theta-\theta_0\|_2,$$
which gives
\begin{align*}
\mL_{n,\theta}(y) & \leq -\frac{1}{2+\max_i |x_i^T(\theta-\theta_0)|} (\theta-\theta_0)^TX^T W X(\theta-\theta_0)\\
& \leq -\frac{\underline{\kappa} \|X\|^2\|\theta-\theta_0\|_2^2}{2+8 s_0^{1/2}\|X\|_\infty \|\theta-\theta_0\|_2} =: -\frac{1}{2} r(\|\theta-\theta_0\|_2)
\end{align*}
for the rate function $r(t) = \downk \|X\|^2 t^2 /(1+4s_0^{1/2}\|X\|_\infty t)$. Thus the event $\check{\mE}_{n,1}(\mathcal{N},r)$ holds deterministically true for any $y\in \{0,1\}^n$ and this choice of $r$. Furthermore, since $g''(t) \leq 1/4$, by considering the remainder in the Taylor expansion of $\ell_{n,\theta}(y)-\ell_{n,\theta_0}(y)$, for $\theta-\theta_0 \in \Theta_0 = \{\theta': S_{\theta'} \subseteq S_{\theta_0}\}$,
$$\mL_{n,\theta}(y) \geq -\tfrac{1}{8}(\theta-\theta_0)^T X^T X(\theta-\theta_0) \geq -\tfrac{1}{8}\overline{\kappa}(s_0)\|X\|^2 \|\theta-\theta_0\|_2^2.$$
Thus $\hat{\mE}_{n,1}(\Theta_0,\bar{L})=\{0,1\}^n$ for $\bar{L}=\overline{\kappa}(s_0)\|X\|^2/4$. Inspection of the proof of Theorem 3 of \cite{atchade2017} shows that he actually only requires the larger event $\mA_{n,1}(\lambda/2) = \{\|\nabla_\theta \ell_{n,\theta_0}(Y)\|_\infty \leq \lambda /2\} \supsetneq \mE_{n,0}(\R^p,\lambda)$ to hold rather than $\mE_{n,0}(\R^p,\lambda)$ (see p. 23  of \cite{atchade2017} - they are incorrectly stated as being equal). In our setting, we may thus replace the event of Theorem 3(1) of \cite{atchade2017} by
$$\mA_{n,1} \big(\lambda/2\big)\cap\check{\mE}_{n,1}(\mathcal{N},r) \cap \hat{\mE}_{n,1}(\Theta_0,\bar{L}) = \mA_{n,1} \big(\lambda/2\big)$$
for $r,\bar{L}$ as above, from which the result follows.

Part (2): \cite{atchade2017} considers the event $\mE_{n,0}(\bar{\Theta}_L,\eta_L) \cap \hat{\mE}_{n,1}(\Theta_0,\bar{L}) \cap \check{\mE}_n(\bar{\Theta}_L,r)$, which we again simplify. From Part (1), we again take $\hat{\mE}_{n,1}(\Theta_0,\bar{L})=\{0,1\}^n$ for $\bar{L}=\overline{\kappa}(s_0)\|X\|^2/4$. Arguing as in Part (1), we get $\check{\mE}_{n,1}(\bar\Theta_L,r) = \{0,1\}^n$ for rate function $r(x) = \tfrac{\downk ((L+1)s_0)\|X\|^2 x^2}{1+\|X\|_\infty \sqrt{(L+1)s_0}x/2}$ using that for $\theta \in \bar\Theta_L$, $|x_i^T(\theta-\theta_0)| \leq \|X\|_\infty \|\theta-\theta_0\|_1 \leq \|X\|_\infty \sqrt{(L+1)s_0}\|\theta-\theta_0\|_2$. For any $\theta \in \theta_0 +  \bar{\Theta}_L$, by Cauchy-Schwarz,
$$|\langle \nabla_\theta \ell_{n,\theta_0}(y),\theta-\theta_0 \rangle| \leq \|\nabla_\theta \ell_{n,\theta_0}(y)\|_\infty \|\theta-\theta_0\|_1 \leq \|\nabla_\theta \ell_{n,\theta_0}(y)\|_\infty \sqrt{(L+2)s_0}\|\theta-\theta_0\|_2,$$
so that $\mE_{n,0}(\bar{\Theta}_L,\eta_L) \supset  \mA_{n,1} \big(\|X\|\sqrt{\log p}\big)$. We hence conclude that
$$\mE_{n,0}(\bar{\Theta}_L,\eta_L) \cap\check{\mE}_{n,1}(\bar\Theta_L,r) \cap \hat{\mE}_{n,1}(\Theta_0,\bar{L})=\mE_{n,0}(\bar{\Theta}_L,\eta_L) \supset  \mA_{n,1} \big(\|X\|\sqrt{\log p}\big)$$
for the above choices of $\bar{L}$, $r$ and $\eta_L$. Applying Theorem 3(2) of \cite{atchade2017} then gives the result.
\end{proof}

The following is the non-asymptotic analogue of Theorem \ref{thm:contraction} in Section \ref{sec:results}.

\begin{theorem}\label{thm:contraction:nonasymp}
Suppose the prior satisfies \eqref{prior_cond} and \eqref{lambda}, and $p^{A_4}\geq 8A_2$. If for $K>0$, the design matrix satisfies condition \eqref{cond:design} with $L=L_K=\frac{K+C}{A_4-\log(4A_2)/\log p}$ and $C$ the constant in Lemma \ref{lem: thm4(1)}, then for any $\theta_0\in \R^p$,
\begin{align*}
E_{\theta_0} Q^*\left( \theta\in \R^p:\|\theta-\theta_0\|_2 \geq  C_K \frac{\sqrt{s_0\log p}}{\|X\| } \right) \leq   \frac{\zeta_n+8e^{ -(K/2)s_0\log p }}{(K/2)s_0\log p } + \frac{2}{p} + \frac{8}{3}p^{-s_0} + 8p^{-Ls_0},
\end{align*}
where  $\zeta_n$ is given in Lemma \ref{lem:KL}, $C_K= \frac{8M_K\sqrt{L_K+2}}{\underline\kappa((L_K+1)s_0)}$, and $M_K= \max(25\alpha, \frac{1+A_3}{16}, \frac{K+\tilde{C}_p}{8L_K}+\frac{1}{8}+\frac{\log (24)}{8\log p})$ with $\tilde{C}_p$  given in Lemma \ref{lem: thm4(2)}.

Furthermore, the mean-squared prediction error $\|p_\theta-p_0\|_n^2=\tfrac{1}{n} \sum_{i=1}^n(\Psi(x_i^T\theta)- \Psi(x_i^T\theta_0))^2$ of the VB posterior $Q^*$ satisfies 
\begin{align*}
E_{\theta_0}Q^*\Big(\theta\in\mathbb{R}^p:\, &\|p_{\theta}-p_0\|_n^2 \geq   \frac{C_K\sqrt{\bar{\kappa}((L_K+1)s_0)}}{4} \sqrt{\frac{s_0\log p}{n} }\Big)\\
&\leq \frac{\zeta_n+8e^{ -(K/2)s_0\log p }}{(K/4)s_0\log p } + 2/p + (8/3)p^{-s_0} + 8p^{-Ls_0}.
\end{align*}
\end{theorem}

\begin{proof}
We first apply Lemma \ref{lem:KL_VB_thm} with 
\begin{align*}
\Theta_n=\Big\{\theta\in\mathbb{R}^p:\, \|\theta-\theta_0\|_2\geq C_K \frac{\sqrt{s_0\log p}}{\|X\| } \Big\},
\end{align*}
$A = \mA_n = \mA_n(t,L,M_1,M_2)$ the event in Lemma \ref{lem:prob}, $\delta_n=K s_0\log p$, and $C=8$. Since $\mA_n \subset \mA_{n,1}(\|X\|\sqrt{\log p})$, Lemma \ref{lem: combined:contraction} implies that the condition \eqref{cond:KL_VB_thm} holds. Using Lemma \ref{lem:KL_VB_thm} followed by Lemma \ref{lem:KL},
\begin{align*}
E_{\theta_0} Q^*\Big(\theta\in\mathbb{R}^p:\,\|\theta-\theta_0\|_2\geq C_K \frac{\sqrt{s_0\log p}}{\|X\| } \Big)1_{\mA_n} &\leq  \frac{E_{\theta_0} \KL (Q^*||\Pi(\cdot|Y))1_{\mA_n} + 8e^{ -\delta_n/2 } }{\delta_n/2}\\
&\leq  \frac{\zeta_n + 8\exp( -(K/2)s_0\log p ) }{(K/2) s_0\log p}.
\end{align*}
By Lemma \ref{lem:prob},
\begin{align*}
E_{\theta_0} Q^*\Big(\theta\in\mathbb{R}^p:\,\|\theta-\theta_0\|_2\geq C_K \frac{\sqrt{s_0\log p}}{\|X\|  }\Big)1_{\mA_n^c}\leq P_{\theta_0}(\mA_n^c)\leq 2/p + (8/3)p^{-s_0} + 8p^{-Ls_0}.
\end{align*}
The first statement follows by combining the above two displays.

Turning to the second statement, Lemma \ref{lem: combined:contraction} implies the condition \eqref{cond:KL_VB_thm} holds with
\begin{align*}
\Theta_n=\{\theta\in\mathbb{R}^p:\, |S_{\theta}|\geq L_Ks_0\},
\end{align*}
$\delta_n=Ks_0\log p$ and $C=2$. Therefore, similarly as above,
\begin{align*}
E_{\theta_0} Q^*\Big( \theta\in\mathbb{R}^p:\,|S_{\theta}|\geq L_Ks_0\Big)1_{\mA_n} \leq  \frac{\zeta_n + 2\exp( -(K/2)s_0\log p ) }{(K/2) s_0\log p}.
\end{align*}
For any $\theta$ in the set in the last display, since $\|\Psi'\|_{\infty}\leq 1/4$,
\begin{align*}
n\|p_\theta-p_{0}\|_n^2\leq\frac{1}{16}\sum_{i=1}^n|x_i^T(\theta-\theta_0)|^2
\leq \frac{1}{16}\|X(\theta-\theta_0)\|_2^2 \leq \frac{1}{16} \bar{\kappa}((L_K+1)s_0)\|X\|^2\|\theta-\theta_0\|_2^2,
\end{align*}
where the last inequality follows from the definition of $\bar\kappa(\cdot)$. The second statement then follows by combining the first statement of the theorem and the last two displays. 
\end{proof}

	\section{Deriving the variational algorithm}\label{sec:algorithm_derivation}

	\subsection{Coordinate ascent equations}

Since the VB minimization problem \eqref{VB} is intractable for Bayesian logistic regression, we instead minimize a surrogate objective obtained by lower bounding the likelihood \cite{bishop2006,jaakkola2000}. This is a standard approach, but we include full details for completeness. For the log-likelihood $\ell_{n,\theta}$ defined in \eqref{likelihood}, it holds that
\begin{equation}\label{lower_bound}
	\begin{split}
	\ell_{n,\theta}(x,y) \geq \sum_{i = 1}^{n} \log \Psi(\eta_i) -\dfrac{\eta_i}{2} + (y_i -\tfrac{1}{2})x_i^T\theta - \dfrac{1}{4\eta_i}\tanh(\eta_i/2)\big((x_i^T\theta)^2 - \eta_i^2\big)=:f(\theta, \eta)
	\end{split}
	\end{equation}
for any $\eta =(\eta_1,\dots,\eta_n)\in \R^n$, see Section \ref{sec:lower bound} for a proof. Hence for any distribution $Q$ for $\theta$,
	 \begin{equation}\label{lb_objective}
	\begin{split}
	\KL (Q||\Pi(\cdot|Y)) &= \int \log\left(\dfrac{dQ(\theta)}{e^{\ell_{n,\theta}(x,y)} d\Pi(\theta)}\right)\ dQ(\theta) + C\\
	&\leq \int \log \dfrac{dQ}{d\Pi}(\theta) - f(\theta, \eta)\ dQ(\theta) + C\\
	&= \KL (Q||\Pi) - E^Q [f(\theta,\eta)] + C,
	\end{split}
	\end{equation}
	where $C$ is independent of $Q$. We minimize the right-hand side over the variational family $Q_{\mu,\sigma,\gamma} \in \Q$, i.e. over the parameters $\mu,\sigma,\gamma$. Since we seek the tightest possible upper bound in \eqref{lb_objective}, we also minimize this over the free parameter $\eta$. In particular, the coordinate ascent variational inference (CAVI) algorithm alternates between updating $\eta$ for fixed $\mu,\sigma,\gamma$ and then cycling through $\mu_j$, $\sigma_j$, $\gamma_j$ and updating these given all other parameters are fixed.
	
Write $E_{\mu,\sigma,\gamma}$ for the expectation when $\theta \sim Q_{\mu,\sigma,\gamma}$. For fixed $\mu,\sigma,\gamma$, update $\eta = (\eta_1,\dots,\eta_n)$ by
\begin{equation}\label{eta}
\eta_i^2 =  E_{\mu,\gamma,\sigma}(x_i^T \theta)^2 =  \sum_{k = 1}^{p} \gamma_k x_{ik}^2(\mu_k^2 + \sigma_k^2) + \sum_{k = 1}^{p}\sum_{l\neq k} (\gamma_kx_{ik}\mu_k)(\gamma_l x_{il}\mu_l),
\end{equation}
see Section \ref{sec:lower bound} for a proof. We now derive the coordinate update equations for $\mu_j, \sigma_j, \gamma_j$ keeping all other parameters, including $\eta$, fixed. For completeness, we allow the Laplace slab to have non-zero mean $\nu$ if desired.
	
\begin{proposition}[Coordinate updates with Laplace prior]
Consider the prior \eqref{prior_beta} with Laplace slab density $g(x) = \tfrac{\lambda}{2}e^{-\lambda|x - \nu|}$, where $\nu\in\R$, $\lambda>0$. Given all other parameters are fixed, the values $\mu_j$ and $\sigma_j$ that minimize \eqref{lb_objective} with $Q = Q_{\mu,\sigma,\gamma}\in \Q$ are the minimizers of the objective functions:
\begin{align*}
\mu_j \mapsto \quad & \lambda \sigma_j \sqrt{\frac{2}{\pi}} e^{-\frac{(\mu_j - \nu)^2}{2\sigma_j^2}} + \lambda(\mu_j-\nu) \mathrm{erf}\bigg( \frac{\mu_j-\nu}{\sqrt{2}\sigma_j}\bigg) + \mu_j^2 \sum_{i=1}^n \frac{1}{4\eta_i} \tanh(\eta_i/2)  x_{ij}^2 \\
& + \mu_j \bigg( \sum_{i=1}^n \frac{1}{2\eta_i} \tanh(\eta_i/2) x_{ij} \sum_{k\neq j} \gamma_k x_{ik} \mu_k  - \sum_{i=1}^n (y_i-1/2)x_{ij} \bigg), \\
\sigma_j \mapsto \quad & \lambda \sigma_j \sqrt{\frac{2}{\pi}} e^{-\frac{(\mu_j - \nu)^2}{2\sigma_j^2}} + \lambda(\mu_j-\nu) \mathrm{erf}\left( \frac{\mu_j-\nu}{\sqrt{2}\sigma_j}\right) -\log \sigma_j +\sigma_j^2  \sum_{i=1}^n \frac{1}{4\eta_i} \tanh(\eta_i/2)  x_{ij}^2,
\end{align*}			
respectively, where $\mathrm{erf}(x) = 2/\sqrt{\pi} \int_{0}^x e^{-t^2} dt$ is the error function. The value $\gamma_j$ that minimizes \eqref{lb_objective} given all other parameters are fixed, is the solution to
\begin{align*}
 - \log \frac{\gamma_j}{1-\gamma_j} & = \log \frac{b_0}{a_0} - \log (\lambda \sigma_j) + \lambda \sigma_j \sqrt{\frac{2}{\pi}} e^{-\frac{(\mu_j - \nu)^2}{2\sigma_j^2}} + \lambda(\mu_j-\nu) \mathrm{erf}\left( \frac{\mu_j-\nu}{\sqrt{2}\sigma_j}\right) - \frac{1}{2} \\
& \quad  - \mu_j \sum_{i=1}^n (y_i-1/2)x_{ij} 
 + \sum_{i=1}^n \frac{1}{4\eta_i} \tanh(\eta_i/2) \bigg( x_{ij}^2 (\mu_j^2 + \sigma_j^2) + 2x_{ij} \mu_j \sum_{k\neq j} \gamma_k x_{ik} \mu_k \bigg).
\end{align*}
\end{proposition}
	
\begin{proof}
Throughout this proof we fix the parameter $\eta\in\R^n$ and let $C$ denote any term constant with respect to the parameters currently being optimized, possibly different on each line. We first compute the update equations for $\mu_j$ and $\sigma_j$ based on \eqref{lb_objective}. We compute $\KL(Q_{\mu,\sigma,\gamma|z_j=1}||\Pi)$, which considers the distribution $Q_{\mu,\sigma,\gamma}$ conditional on $z_j=1$, as a function of $(\mu_j,\sigma_j)$, holding all other parameters fixed. Using that $Q_{\mu,\sigma,\gamma}$ is a factorizable distribution and that conditional on $z_j=1$ the variational distribution of $\theta_j$ is singular to the Dirac measure $\delta_0$, we can simplify $\tfrac{dQ_{\mu,\sigma,\gamma|z_j=1}}{d\Pi} = C \tfrac{dQ_{\mu_j,\sigma_j|z_j=1}}{d\Pi_j^c}$, where $\Pi_j^c$ is the continuous part of the prior distribution for $\theta_j$ and $C$ does not depend on $\mu_j$ or $\sigma_j$. Recall that $\Pi_j^c = \int_0^1 w \text{Lap}(\nu,\lambda) dw = a_0/(a_0+b_0) \text{Lap}(\nu,\lambda)$. Thus for $\phi_{\mu,\sigma}$ the density of a $N(\mu,\sigma^2)$ distribution and $\overline{w}=a_0/(a_0+b_0)$,
$$\log \frac{dQ_{\mu,\sigma,\gamma|z_j=1}}{d\Pi}(\theta) = \log \frac{dQ_{\mu_j,\sigma_j|z_j=1}}{d\Pi_j^c}(\theta_j) + C = \log \frac{\phi_{\mu_j,\sigma_j}(\theta_j)}{\overline{w}g_j(\theta_j)} + C.$$
Taking expectations with respect to $Q_{\mu,\sigma,\gamma|z_j=1}$, 
 \begin{equation*}
		\begin{split}
		E_{\mu,\sigma,\gamma \mid z_j = 1} \left[ \log \frac{\phi_{\mu_j,\sigma_j}(\theta_j)}{\overline{w}g_j(\theta_j)} \right]
		&= E_{\mu,\sigma\mid z_j = 1}\left[ -\log (\lambda \sigma_j) - \dfrac{(\theta_j - \mu_j)^2}{2\sigma_j^2} + \lambda |\theta-\nu|\right] + C\\
		&= -\log (\lambda\sigma_j) + \lambda E_{\mu,\sigma\mid z_j = 1}|\theta_j-\nu| + C,
		\end{split}
		\end{equation*} 
		where we have used $E_{\mu,\sigma,\gamma \mid z_j = 1} [(\theta_j-\mu_j)^2/\sigma_j^2] = 1$. Under the variational distribution, $\lambda |\theta_j - \nu|$ follows a folded Gaussian distribution, hence
 \begin{equation*}
		\lambda E_{\mu_j,\sigma_j\mid z_j = 1}|\theta_j-\nu| = \lambda \sigma_j \sqrt{\frac{2}{\pi}} e^{-\frac{(\mu_j - \nu)^2}{2\sigma_j^2}} + \lambda(\mu_j-\nu) \mathrm{erf}\left( \frac{\mu_j-\nu}{\sqrt{2}\sigma_j}\right).
		\end{equation*} 
Combining the last three displays gives $\KL(Q_{\mu,\sigma,\gamma|z_j=1}||\Pi)$ as a function of $\mu_j,\sigma_j$. Using this expression and evaluating $E_{\mu,\sigma,\gamma|z_j=1}\big[f(\theta,\eta)\big]$ using Lemma \ref{lem:VB_expectations} below, the upper bound in \eqref{lb_objective} equals, as a function of $\mu_j,\sigma_j$,
\begin{align*}
& -\log (\lambda\sigma_j) + \lambda \sigma_j \sqrt{\frac{2}{\pi}} e^{-\frac{(\mu_j - \nu)^2}{2\sigma_j^2}} + \lambda(\mu_j-\nu) \mathrm{erf}\left( \frac{\mu_j-\nu}{\sqrt{2}\sigma_j}\right) \\
& - \sum_{i=1}^n (y_i-1/2)x_{ij}\mu_j + \sum_{i=1}^n \frac{1}{4\eta_i} \tanh(\eta_i/2) \left( x_{ij}^2 (\mu_j^2 + \sigma_j^2) + 2x_{ij} \mu_j \sum_{k\neq j} \gamma_k x_{ik} \mu_k \right) +C,
\end{align*}		
where $C$ is independent of $\mu_j,\sigma_j$. Minimizing the display with respect to either $\mu_j$ or $\sigma_j$ gives the desired result.

For updating the inclusion probabilities $\gamma_j$, we proceed as above without conditioning on $z_j = 1$. Keeping track of only the $\gamma_j$ terms,
\begin{align*}
E_{\mu,\sigma,\gamma} \left[ \log \frac{dQ_{\mu,\sigma,\gamma}}{d\Pi}(\theta)\right] & = E_{\mu,\sigma,\gamma} \left[ \log \frac{d(\gamma_j N(\mu_j,\sigma_j^2) + (1-\gamma_j)\delta_0)}{d(\overline{w}\text{Lap}(\nu,\lambda) + (1-\overline{w})\delta_0)}(\theta_j) \right] + C \\
& = E_{\mu,\sigma,\gamma} \left[ 1_{\{z_j=1\}} \log \frac{\gamma_j dN(\mu_j,\sigma_j^2)}{\overline{w}d\text{Lap}(\nu,\lambda)}(\theta_j) + 1_{\{z_j=0\}}\log \frac{1-\gamma_j}{1-\overline{w}} \right] + C \\
& = \gamma_j E_{\mu,\sigma,\gamma|z_j=1} \left[ \log \frac{\phi_{\mu_j,\sigma_j}(\theta_j)}{g_j(\theta_j)}  \right] + \gamma_j \log \frac{\gamma_j}{\overline{w}}+ (1-\gamma_j) \log \frac{1-\gamma_j}{1-\overline{w}}  + C.
\end{align*}
The first expectation was evaluated above. Using this and evaluating $E_{\mu,\sigma,\gamma}\big[f(\theta,\eta)\big]$ using Lemma \ref{lem:VB_expectations} below, the upper bound in \eqref{lb_objective} equals, as a function of $\gamma_j$,
\begin{align*}
& \gamma_j \Bigg\{ -\log (\lambda\sigma_j)+ \lambda \sigma_j \sqrt{\frac{2}{\pi}} e^{-\frac{(\mu_j - \nu)^2}{2\sigma_j^2}} + \lambda(\mu_j-\nu) \mathrm{erf}\left( \frac{\mu_j-\nu}{\sqrt{2}\sigma_j}\right) - \mu_j \sum_{i=1}^n (y_i-1/2)x_{ij} \\
& \qquad - \frac{1}{2} + \sum_{i=1}^n \frac{1}{4\eta_i} \tanh(\eta_i/2) \left( x_{ij}^2 (\mu_j^2 + \sigma_j^2) + 2x_{ij} \mu_j \sum_{k\neq j} \gamma_k x_{ik} \mu_k \right) \Bigg\} \\
& \qquad + \gamma_j \log \frac{\gamma_j}{\overline{w}}+ (1-\gamma_j) \log \frac{1-\gamma_j}{1-\overline{w}}  + C,
\end{align*}
where $C$ is independent of $\gamma_j$. As a function of $\gamma_j$, this takes the form
$$h(\gamma_j) = \gamma_j \log\frac{\gamma_j}{a} + (1-\gamma_j)\log\dfrac{1-\gamma_j}{b} + c\gamma_j,$$
with $a,b\in (0,1)$ and $c\in \R$. By differentiating, $h$ is convex and has a global minimizer $\bar{\gamma}_j \in [0,1]$ satisfying
$$-\log \frac{\bar{\gamma}_j}{1-\bar{\gamma}_j} = c + \log \frac{b}{a}.$$
Substituting in the above values for $a,b,c$ gives the result.
\end{proof}

\subsection{Variational lower bound}\label{sec:lower bound}
	
We now derive the lower bound \eqref{lower_bound} as in \cite{jaakkola2000}. Recall the log-likelihood \eqref{likelihood}:
\begin{equation*}
\ell_{n,\theta}(x,y) = \sum_{i=1}^n y_i x_i^T \theta - g(x_i^T \theta),
\end{equation*}	
where $g(t) = \log (1+e^t)$, $t\in \R$. We lower bound the second term above using a Taylor expansion in $x^2$. The following lemma fills in some details from \cite{jaakkola2000}, where this technique was proposed.
	
\begin{lemma}
For $\Psi(x) = (1+e^{-x})^{-1}$ the standard logistic function and any $\eta\in\R$,
$$\log \Psi(x) \geq \dfrac{x - \eta}{2} + \log \Psi(\eta) - \dfrac{1}{4\eta}\tanh(\eta/2)(x^2 - \eta^2).$$
\end{lemma}
	
\begin{proof}
Note that we can write $\log \Psi(x) = x/2 - \log (e^{x/2} + e^{-x/2}).$
By elementary calculations, it can be shown that the second term on the right hand side is convex in the variable $x^2$. We can therefore use its first order Taylor approximation in $x^2$ to derive a lower bound for $\log\Psi(x)$, i.e. for any $\eta \in \R$,
 \begin{equation*}
		\begin{split}
		\log \Psi(x) &\geq \frac{x}{2} - \log(e^{\eta/2} + e^{-\eta/2}) - \frac{1}{4\eta}\tanh(\eta/2)(x^2 - \eta^2)\\
		&= \frac{x - \eta}{2} + \log \Psi(\eta) - \frac{1}{4\eta}\tanh(\eta/2)(x^2 - \eta^2).
		\end{split}
		\end{equation*}
	\end{proof}
	
	Substituting this lower bound into each term $-g(x_i^T\theta) = \log \Psi(-x_i^T \theta)$ in the log-likelihood yields \eqref{lower_bound}.
	
We next obtain the update equation \eqref{eta} for the free parameter $\eta\in \R^n$. Minimizing \eqref{lb_objective} over $\eta$ for fixed $Q = Q_{\mu,\sigma,\gamma} \in \Q$ is equivalent to solving
\begin{equation*}
\widetilde{\eta} = \argmax_{\eta\in\R^n} E_{\mu,\sigma,\gamma}\big[f(\theta, \eta)\big].
\end{equation*}
	
\begin{lemma}
The function $f_a:\R\to\R$, $a\geq 0$, given by
$$f_a(x) = \log \Psi(x) - \dfrac{x}{2} - \dfrac{1}{4x}\tanh(x/2)\big(a^2 - x^2\big)$$
is symmetric about zero and possesses unique maximizers at $x = \pm a$.
\end{lemma}
	
	\begin{proof}
		Indeed, $\tanh(-x) = -\tanh(x)$ shows that the last term is symmetric around zero. Moreover, $\Psi(-x) = 1-\Psi(x)$ yields $\log \Psi(-x) + x/2= \log \Psi(x) - x/2$ thereby proving symmetry of $f_a$. Using $2\Psi(x) - 1 = \tanh(x/2)$,
\begin{align*}
		f'_a(x) &= \frac{\Psi'(x)}{\Psi(x)} - \frac{1}{2} + \frac{1}{4}\left(\tanh(\nicefrac{x}{2}) + \frac{x}{2(\cosh \nicefrac{x}{2})^2}\right) - \left( \frac{x (\cosh \nicefrac{x}{2})^{-2} - 2\tanh(\nicefrac{x}{2})}{8x^2}\right) a^2\\
		&= (a^2 - x^2)\left(\dfrac{\tanh(\nicefrac{x}{2})}{4x^2} - \dfrac{1}{8x (\cosh \nicefrac{x}{2})^2}\right).
\end{align*}
		Since $f_a'(\pm a) = 0$, it remains to show $f''_a(\pm a) < 0$. Note that \begin{equation*}
		\begin{split}
		f''_a(x) = (a^2 - x^2) \frac{d}{dx}\left(\frac{\tanh(\nicefrac{x}{2})}{4x^2} - \frac{1}{8x (\cosh \nicefrac{x}{2})^2}\right) - \frac{\tanh(\nicefrac{x}{2})}{2x} + \frac{1}{4 (\cosh \nicefrac{x}{2})^2}.
		\end{split}
		\end{equation*} The first term vanishes at $x = \pm a$ and the second is symmetric about zero, so the formula $\sinh(x)\cosh(x) = \sinh(2x)/2$ yields 
\begin{equation*}
		f''_a(\pm a) = -\dfrac{2\sinh(\nicefrac{a}{2})\cosh(\nicefrac{a}{2}) - a}{4a(\cosh\nicefrac{a}{2})^2}
		= -\dfrac{\sinh(a) - a}{4a(\cosh\nicefrac{a}{2})^2}.
		\end{equation*} This concludes the proof as $\sinh(x)/x \geq 1$.
	\end{proof}
By the last lemma, we can restrict the free parameter to $\eta\in \R_{\geq 0}^n$ and take $\widetilde{\eta}_i^2 = E_{\mu,\gamma,\sigma}(x_i^T \theta)^2$ to maximize $E_{\mu,\sigma,\gamma}f(\theta,\eta)$. The update \eqref{eta} then follows from the following lemma.
	
	\begin{lemma}\label{lem:VB_expectations}
		For $Q_{\mu,\sigma,\gamma}\in\Q$,
		\begin{equation*}
		\begin{split}
		\bbE_{\mu,\sigma,\gamma}[x_i^T \theta] &= \sum_{k=1}^{p} \gamma_k \mu_k x_{ik},\\
		\bbE_{\mu,\sigma,\gamma}\big[(x_i^T \theta)^2\big] &= \sum_{k = 1}^{p} \gamma_k x_{ik}^2(\mu_k^2 + \sigma_k^2) + \sum_{k = 1}^{p}\sum_{l\neq k} (\gamma_kx_{ik}\mu_k)(\gamma_l x_{il}\mu_l).
		\end{split}
		\end{equation*}
If the expectations are instead taken over $E_{\mu,\sigma,\gamma|z_j=1}$, then the same formulas hold true with $\gamma_j=1$.
	\end{lemma}
	
	\begin{proof}
		Since $\theta_k \sim^{iid}(1-\gamma_k)\delta_0+ \gamma_k\mathcal{N}(\mu_k, \sigma_k^2)$ under $Q_{\mu,\sigma,\gamma}$, the first claim follows by linearity of the expectation. Using that $\theta_k = \theta_k 1_{\{z_k = 1\}}$, $Q_{\mu,\sigma,\gamma}$-almost surely, and that $(\theta_k)$ are independent under the mean-filed distribution $Q_{\mu,\sigma,\gamma}$,
 \begin{align*}
 E_{\mu,\sigma,\gamma}\big[(x_i^T \theta)^2\big] & = E_{\mu,\sigma,\gamma} \left(\sum_{k = 1}^{p} x_{ik}\theta_k 1_{\{z_k =1\}}\right)^2 \\
 & = \sum_{k = 1}^{p} \gamma_kx_{ik}^2(\mu_k^2 + \sigma_k^2) +  \sum_{k = 1}^{p}\sum_{l\neq k} (\gamma_kx_{ik}\mu_k)(\gamma_l x_{il}\mu_l).
 \end{align*}
	\end{proof}

\bibliography{references}{}
\bibliographystyle{acm}

\end{document}